%% file: example_paper.tex
\documentclass{article}

\usepackage{microtype}
\usepackage{graphicx}
\usepackage{subfigure}
\usepackage{booktabs} 
\usepackage{float}
\usepackage{hyperref}

\usepackage[accepted]{icml2025}

\usepackage{amsmath}
\usepackage{amssymb}
\usepackage{mathtools}
\usepackage{amsthm}

\usepackage[capitalize,noabbrev]{cleveref}

\theoremstyle{plain}
\newtheorem{theorem}{Theorem}[section]

\theoremstyle{definition}
\newtheorem{definition}[theorem]{Definition}

\theoremstyle{remark}
\newtheorem{remark}[theorem]{Remark}

\usepackage[textsize=tiny]{todonotes}

\usepackage{amsfonts}
\usepackage{amssymb}
\usepackage{bm}
\usepackage{amsmath}
\usepackage{amsthm}
\usepackage{stmaryrd}
\usepackage{color}
\usepackage{xcolor}
\usepackage[subnum]{cases}
\usepackage{rotating}
\usepackage{enumitem}
\usepackage{accents}
\usepackage[title]{appendix}
\usepackage{mathtools}
\usepackage{url}
\usepackage{rotating}
\usepackage{framed}
\usepackage{lscape}
\usepackage{multirow}
\usepackage{latexsym}
\usepackage{mathrsfs}
\usepackage[new]{old-arrows}
\usepackage{array}
\usepackage{makecell}
\usepackage{float}
\usepackage{tabularray}
\usepackage{tablefootnote}
\usepackage{tikz}
\usepackage[all]{xy}
\usepackage{tabularx}
\usepackage{tikz-cd}
\usepackage[usestackEOL]{stackengine}
\usepackage{pdflscape}
\usepackage{adjustbox}
\usepackage{booktabs}
\usepackage{algorithmic}
\usepackage{algorithm}
\usepackage{caption}

\usepackage{subcaption}

\usepackage{titletoc}

\renewcommand{\le}{\leqslant}
\renewcommand{\ge}{\geqslant}

\renewcommand{\geq}{\geqslant}

\newcommand{\OO}{\operatorname{O}}

\usepackage[normalem]{ulem}
\newcommand\tsout{\bgroup\markoverwith{\textcolor{red}{\rule[0.5ex]{2pt}{0.4pt}}}\ULon}

\icmltitlerunning{Tree-Sliced Wasserstein Distance with Nonlinear Projection}

\begin{document}

\twocolumn[
\icmltitle{Tree-Sliced Wasserstein Distance with Nonlinear Projection}



\icmlsetsymbol{equal}{*}
\icmlsetsymbol{co}{$\dagger$}

\begin{icmlauthorlist}
\icmlauthor{Thanh Tran}{equal,vinuni}
\icmlauthor{Viet-Hoang Tran}{equal,nus}
\icmlauthor{Thanh Chu}{nus}
\icmlauthor{Trang Pham}{qualcomm}
\icmlauthor{Laurent El Ghaoui}{co,vinuni}
\icmlauthor{Tam Le}{co,ism}
\icmlauthor{Tan M. Nguyen}{co,nus}
\end{icmlauthorlist}

\icmlaffiliation{vinuni}{VinUniversity}
\icmlaffiliation{nus}{National University of Singapore}
\icmlaffiliation{qualcomm}{Movian AI}
\icmlaffiliation{ism}{The Institute of Statistical Mathematics}

\icmlcorrespondingauthor{Viet-Hoang Tran}{hoang.tranviet@u.nus.edu}

\icmlkeywords{tree-sliced Wasserstein distance, optimal transport}

\vskip 0.3in
]



\printAffiliationsAndNotice{$^*$Equal contribution $\,^{\dagger}$Co-last authors}

\begin{abstract}
Tree-Sliced methods have recently emerged as an alternative to the traditional Sliced Wasserstein (SW) distance, replacing one-dimensional lines with tree-based metric spaces and incorporating a splitting mechanism for projecting measures. This approach enhances the ability to capture the topological structures of integration domains in Sliced Optimal Transport while maintaining low computational costs. Building on this foundation, we propose a novel nonlinear projectional framework for the Tree-Sliced Wasserstein (TSW) distance, substituting the linear projections in earlier versions with general projections, while ensuring the injectivity of the associated Radon Transform and preserving the well-definedness of the resulting metric. By designing appropriate projections, we construct efficient metrics for measures on both Euclidean spaces and spheres. Finally, we validate our proposed metric through extensive numerical experiments for Euclidean and spherical datasets. Applications include gradient flows, self-supervised learning, and generative models, where our methods demonstrate significant improvements over recent SW and TSW variants. The code is publicly available at \url{https://github.com/thanhqt2002/NonlinearTSW}.
\end{abstract}

\input{sections/intro_relate}

\input{sections/mainbody}
\input{sections/experiment}

\input{sections/conclusion}



\section*{Acknowledgements}

We thank the area chairs and anonymous reviewers for their comments. TL gratefully acknowledges the support of JSPS KAKENHI Grant number 23K11243, and Mitsui Knowledge Industry Co., Ltd. grant.



\section*{Impact Statement}
This paper presents work whose goal is to advance the field of 
Machine Learning. There are many potential societal consequences 
of our work, none which we feel must be specifically highlighted here.



\bibliography{example_paper}
\bibliographystyle{icml2025}

\newpage
\onecolumn
\begin{table}[ht]
\renewcommand*{\arraystretch}{1.2}
    \begin{tabularx}{\textwidth}{p{0.40\textwidth}X}
    $\mathbb{R}^d$ & $d$-dimensional Euclidean space \\
    $\mathbb{S}^{d}$ & $d$-dimensional hypersphere \\
    $\|\cdot\|_2$ & Euclidean norm \\
    $\left< \cdot, \cdot \right>$ & standard dot product \\
    $\mathbb{S}^{d-1}$ & $(d-1)$-dimensional hypersphere \\
    $\theta$ & unit vector \\
    $\sqcup$ & disjoint union \\
        $L^1(X)$ & space of Lebesgue integrable functions on $X$ \\
    $\mathcal{P}(X)$ & space of probability distributions (or measures) on $X$ \\
        $\mu,\nu$ & measures \\
        $\delta(\cdot)$ & $1$-dimensional Dirac delta function \\
        $\mathcal{U}(\mathbb{S}^{d-1})$ & uniform distribution on $\mathbb{S}^{d-1}$ \\
        $\sharp$ & pushforward (measure) \\
    $\mathcal{C}(X,Y)$ & space of continuous maps from $X$ to $Y$ \\
        $d(\cdot,\cdot)$ & metric in metric space \\
        $\operatorname{T}(d)$ & translation group of order $d$ \\
        $\operatorname{O}(d)$ & orthogonal group of order $d$ \\
    $\operatorname{E}(d)$ & Euclidean group of order $d$ \\
    $g$ & element of group \\
    $\text{W}_p$ & $p$-Wasserstein distance \\
    $\text{SW}_p$ & Sliced $p$-Wasserstein distance \\
    $\mathcal{L}$ & system of lines, tree system \\
    $r^x_y$ & spherical ray \\
$\mathcal{T}, \mathcal{T}^x_{y_1,\ldots,y_k}$ & spherical tree \\
    $\mathbb{L}^d_k$ & space of symtems of $k$ lines in $\mathbb{R}^d$ \\
    $L$ & number of tree systems \\
    $k$ & number of lines in a system of lines or a tree system \\
    $\mathcal{R}$ & original Radon Transform \\
    $\mathcal{R}^\alpha$ & Radon Transform on Systems of Lines, or Radon Transform on Spherical Trees \\
    $\Delta_{k-1}$ & $(k-1)$-dimensional standard simplex \\
    $\alpha$ & splitting map \\
    $\delta$ & Dirac delta function \\
    $\mathbb{T}$ & space of tree systems \\
    $\sigma$ & distribution on space of tree systems
     \end{tabularx}
\end{table}
\newpage
\appendix
\onecolumn

\input{sections/appendix}


\end{document}

%% file: sections/intro_relate.tex
\section{Introduction}
Optimal Transport (OT) \citep{villani2008optimal, peyre2019computational} and Sliced-Wasserstein (SW) \citep{rabin2011wasserstein, bonneel2015sliced} provide geometrically meaningful metrics in the space of probability measures, making them widely applicable across various fields. These include machine learning \citep{nguyen2021optimal, bunne2022proximal, hua2023curved, fan2022complexity, le2024generalized, le2024noisy, kessler2025sava, chapel2025differentiable, chapel2025one}, multimodal data analysis~\citep{park2024bridging, luong2024revisiting}, computer vision and graphics~\citep{lavenant2018dynamical, nguyen2021point, Saleh_2022_CVPR, rabin2011wasserstein, solomon2015convolutional, vu2025few}, statistics~\citep{mena2019statistical, pmlr-v99-weed19a, pmlr-v151-wang22f, pham2024scalable, liu2022entropy, nguyen2022many, nietert2022outlier}. By utilizing the closed-form solution of one-dimensional OT problems, SW substantially reduces the computational complexity typically associated with OT \citep{rabin2011wasserstein, bonneel2015sliced, peyre2019computational}.


\paragraph{Related work.} Several enhancements have been proposed to improve different aspects of the SW framework, including optimizing the sampling process \citep{nguyen2020distributional, nadjahi2021fast, nguyen2024quasimonte}, selecting optimal projection directions \citep{deshpande2019max}, and refining the projection mechanism \citep{kolouri2019generalized, chen2022augmented, bonet2023hyperbolic}.

The Tree-Sliced framework \cite{tran2024tree, tran2025spherical, tran2025distancebased} has recently emerged as an alternative to the traditional SW framework by replacing one-dimensional projection lines with more structured domains, known as tree systems. These systems function similarly to lines but incorporate a more sophisticated and interconnected structure. Instead of projecting measures onto individual lines, this method distributes them across multiple linked lines, forming a hierarchical structure. This approach enhances the representation of topological information while preserving the computational efficiency by leveraging the closed-form solution of OT problems in tree-metric spaces \cite{indyk2003fast, le2019tree}. However, Tree-Sliced frameworks are still restricted to linear projections, as the integration domains used in \cite{tran2024tree, tran2025distancebased} remain confined to hyperplanes. In contrast, several advanced projection techniques have been investigated for SW, enhancing its flexibility and effectiveness \citep{kuchment-2006, kolouri2019generalized, chen2022augmented, bonet2023hyperbolic}. Building on these advancements, this paper introduces a novel framework for the tree-sliced method that integrates nonlinear projections, further broadening its scope and improving its performance.

\textbf{Contribution.} Our contributions are three-fold:
\begin{itemize}

    \item We introduce the Generalized Radon Transform and Spatial Radon Transform on Systems of Lines, extending previous variants that were limited to linear projections. Along with these extensions, we provide theoretical results and proofs demonstrating their injectivity. Furthermore, we generalize the conventional Euclidean framework by proposing a spherical tree-sliced version tailored for nonlinear Radon transforms.
    \item We introduce Tree-Sliced distances based on the proposed Radon Transform, namely the Circular Tree-Sliced Wasserstein distance and the Spatial Tree-Sliced Wasserstein distance, along with an extended version for the spherical setting, called the Spatial Spherical Tree-Sliced Wasserstein distance. Furthermore, we examine different choices of functions that define the nonlinear projections, analyze their computational complexity and explain why certain choices lead to more efficient metrics.
    \item We assess the proposed metrics across various tasks, including gradient flows and generative models, on both Euclidean and spherical data, highlighting their practical effectiveness and computational efficiency.
\end{itemize}

\textbf{Organization.} The structure of the paper is as follows: Section \ref{sec:preliminaries} provides an overview of different variants of the Wasserstein distance, while Section \ref{sec:Radon Transform with Nonlinear Projection} explores various forms of the Radon Transform with Nonlinear Projection used to derive corresponding Wasserstein distances. Section \ref{sec:Radon Transform on Systems of Lines with Nonlinear Projection} introduces novel Radon Transforms on Systems of Lines with Nonlinear Projection and examines their injectivity. In Section \ref{sec:Tree-Sliced Wasserstein Distance with Nonlinear Projection}, two new Tree-Sliced Wasserstein distances associated with the proposed transform are introduced, along with an analysis of their fundamental components. Finally, Section \ref{sec:Experimental Results} assesses the performance of the proposed methods on both Euclidean and Spherical data. Additional materials related to the spherical settings of the proposed method, as well as background, theoretical foundations, and supplementary content, are provided in the Appendix. 

%% file: sections/mainbody.tex
\section{Preliminaries}
\label{sec:preliminaries}
This section reviews the Wasserstein distance between measures and its various sliced variants. For simplicity, the focus is on measures with a finite first moment, while measures with a finite $p^\text{th}$-moment are treated analogously.

\paragraph{Wasserstein distance.}  Given a measurable space $\Omega$ endowed with a metric $d$. Let $\mu,\nu \in \mathcal{P}(\Omega)$, and $\mathcal{P}(\mu, \nu)$ be the set of distributions $\pi$ coupling between $\mu$ and $\nu$. The Wasserstein distance ($\text{W}$) \citep{villani2008optimal} between $\mu$, $\nu$ is:
\begin{equation}\label{equation:OTproblem}
\text{W}(\mu, \nu) = \underset{\pi \in \mathcal{P}(\mu, \nu)}{\inf}   \int_{\Omega \times \Omega} d(x,y) ~ d\pi(x,y) .
\end{equation}

\paragraph{Sliced Wasserstein distance.} The Radon Transform \citep{helgason2011radon} $\mathcal{R} ~ \colon ~ L^1(\mathbb{R}^d) \rightarrow L^1(\mathbb{R} \times \mathbb{S}^{d-1})$ is:
\begin{align}\label{eq:traditional-Radon-Transform}
    \mathcal{R}f(t,\theta) = \int_{\mathbb{R}^d} f(y) \cdot \delta(t-\left<y,\theta\right>) ~dy,
\end{align}
where $\delta$ is the Dirac delta function. 

The Sliced Wasserstein distance (SW)~\citep{rabin2011wasserstein, bonneel2015sliced} between $\mu,\nu \in \mathcal{P}(\mathbb{R}^d)$ is:
\begin{align}
\label{eq:SW}
    \text{SW}(\mu,\nu)  =  \int_{\mathbb{S}^{d-1}} \text{W} (\mathcal{R}f_{\mu}(\cdot, \theta), \mathcal{R}f_{\nu}(\cdot, \theta)) ~ d\sigma(\theta),
\end{align}
where $\sigma = \mathcal{U}(\mathbb{S}^{d-1})$ is the uniform distribution on the sphere, and $f_\mu, f_\nu$ are the probability density functions of $\mu, \nu$, respectively. The one-dimensional Wasserstein distance in Eq.~(\ref{eq:SW}) has the closed-form $\text{W}(\theta \sharp \mu,\theta \sharp \nu) = 
     \int_0^1 |F_{\mathcal{R}f_{\mu}(\cdot, \theta)}^{-1}(z) - F_{\mathcal{R}f_{\nu}(\cdot, \theta)}^{-1}(z)| dz $,
where $F_{\mathcal{R}f_{\mu}(\cdot, \theta)}$, $F_{\mathcal{R}f_{\nu}(\cdot, \theta)}$  are the cumulative distribution functions of $\mathcal{R}f_{\mu}(\cdot, \theta)$, $\mathcal{R}f_{\nu}(\cdot, \theta)$, respectively. The Monte Carlo method is employed to approximate the intractable integral in Eq.~(\ref{eq:SW}):
\begin{align}
\label{eq:MCSW}
    \widehat{\text{SW}}(\mu,\nu) = \dfrac{1}{L} \sum_{i=1}^L\text{W}(\mathcal{R}f_{\mu}(\cdot, \theta_i), \mathcal{R}f_{\nu}(\cdot, \theta_i)) ,
\end{align}
where $\theta_{1},\ldots,\theta_{L}$ are drawn independently from $\sigma$.


\paragraph{Tree-Sliced Wasserstein distance on Systems of Lines.}  Rather than projecting functions onto lines, i.e., directions in $\mathbb{S}^{d-1}$, as in the original Radon Transform, \cite{tran2024tree, tran2025distancebased} introduces an alternative approach that replaces lines with different metric measure spaces, known as tree systems. These tree systems are formed by connecting multiple copies of real lines, creating a structured space. Functions are then partitioned using a splitting mechanism and projected onto these spaces, effectively capturing the positional information of both measure supports and slices. Further details on this Tree-Sliced framework can be found in Appendix~\ref{appendix:sec:background-tsw-sl} and Appendix~\ref{appendix:background-stsw}.

\section{Radon Transform with Nonlinear Projection}
\label{sec:Radon Transform with Nonlinear Projection}
In this section, we explore nonlinear projections utilized in various existing Sliced Wasserstein variants, and compare them with the traditional linear projection in the original Radon Transform that leads to the Sliced Wasserstein distance. Based on these observations, we provide an overview of how these nonlinear projections can be incorporated into the framework of the Radon Transform on Systems of Lines.

\subsection{Generalized and Spatial Radon Transforms}
\label{subsection:Circular and Spatial Radon Transforms}
Let $\Psi$ be a set of feasible parameter, the Generalized Radon Transform $\mathcal{G}$ (GRT) \cite{kolouri2019generalized} is defined by:
\begin{align}\label{eq:GRT}
    &~~~~~\mathcal{G} ~ \colon ~ L^1(\mathbb{R}^d) \longrightarrow L^1(\mathbb{R} \times \Psi), \notag \\
    &\text{s.t.} ~~ \mathcal{G}f(t, \psi) = \int_{\mathbb{R}^d} f(y) \cdot \delta(t - g(y,\psi)) ~ dy.
\end{align}
Here, $g \colon \mathbb{R}^d \times \Omega \rightarrow \mathbb{R}$ is called the defining function of $\mathcal{G}$. The GRT of $f \in L^1(\mathbb{R}^d)$ is the integration of $f$ over hypersurfaces $\{y \in \mathbb{R}^d ~ \colon ~ t = g(y,\psi)\}$ for $t \in \mathbb{R}, \psi \in \Psi$. One common choice for the defining function occurs when $\Psi = \mathbb{R}_{\ge 0} \times \mathbb{S}^{d-1}$, and $g$ is defined as  the circular function:
\begin{align} \label{eq:circular-defining-f}
    g(y,r,\theta) = \|y - r\theta\|_2 ~ , ~ \forall y \in \mathbb{R}^d, (r,\theta) \in \Psi. 
\end{align}
This choice results in the Circular Radon Transform (CRT) \cite{kuchment-2006}. Another possible defining function, based on homogeneous polynomials of odd degree, is presented in \cite{rouvire-2015}. However, this represents a special case of the next Radon Transform we discuss. Given a positive integer $d_\theta$, the Spatial Radon Transform $\mathcal{H}$ (SRT) \cite{chen2022augmented} is defined by:
\begin{align}\label{eq:SRT}
    &~~~~~\mathcal{H} ~ \colon ~ L^1(\mathbb{R}^d) \longrightarrow L^1(\mathbb{R} \times \mathbb{S}^{d_\theta -1}), \notag \\
    &\text{s.t.} ~~ \mathcal{H}f(t, \theta) = \int_{\mathbb{R}^d} f(y) \cdot \delta(t - \left<h(y),\theta\right>) ~ dy,
\end{align}
where $h\colon \mathbb{R}^d \rightarrow \mathbb{R}^{d_\theta}$ is an injective continuous map. The SRT of $f \in L^1(\mathbb{R}^d)$ is defined as the integration of $f$ over the hypersurfaces given by $\{y \in \mathbb{R}^d ~ \colon ~ t = \left <h(y),\theta \right>\}$ for $t \in \mathbb{R}, \theta \in \mathbb{S}^{d_\theta-1}$.
\begin{remark}
    When $\Omega_\theta = \mathbb{S}^{d-1}$, $g(y,\theta) = \left<y,\theta \right>$ in Eq.~\eqref{eq:GRT}, or $d = d_\theta, h(y) = y$ in Eq.~\eqref{eq:SRT}, it is clear that the GRT and SRT recover the original Radon Transform in Eq.~\eqref{eq:traditional-Radon-Transform}.
\end{remark}
By definition, the SRT is a special case of the GRT, but these two transforms can be interpreted from two distinct perspectives. For a general function $g$, the GRT represents a projection along hypersurfaces in $\mathbb{R}^d$, defined by the level sets of $g$. In contrast, for a general function $h$, the SRT involves mapping functions from $\mathbb{R}^d$ to a new space $\mathbb{R}^{d_\theta}$, where the Radon Transform is then applied.

\paragraph{Well-definedness and injectivity.} Certain conditions on $g$ in GRT and $h$ in SRT are necessary to ensure the well-definedness of GRT and SRT. Additionally, since injectivity is typically required for Radon Transform variants, specific assumptions on $g$ and $h$ are made to achieve this property. However, since these properties, along with the inverse problem related to Radon Transform variants, remain long-standing research questions \cite{beylkin1984inversion, ehrenpreis2003universality, uhlmann2003inside, homan2017injectivity} and fall beyond the scope of this paper, we restrict our discussion to mentioned examples of $g$ and $h$ found in the literature.

\subsection{Incorporating Nonlinear Projectional Framework into Systems of Lines Setting}
Here, we provide a brief overview of the Radon Transform on Systems of Lines (RTSL) to highlight the potential of integrating the nonlinear approach discussed in Section~\ref{subsection:Circular and Spatial Radon Transforms}. The formal construction of RTSL with notations is detailed in Appendix~\ref{appendix:sec:background-tsw-sl} and in \cite{tran2024tree, tran2025distancebased}. Roughly speaking, unlike the traditional Radon Transform in Eq.~\eqref{eq:traditional-Radon-Transform}, which projects onto one line at a time, RTSL extends this concept by simultaneously projecting onto a set of interconnected multiple lines, denoted as $\mathcal{L}$, using a splitting mechanism:
    \begin{align} \label{eq:rtsl-1}
        &\mathcal{R}^{\alpha} \colon L^1(\mathbb{R}^d) \rightarrow \prod_{\mathcal{L} \in \mathbb{L}^d_k} L^1(\mathcal{L}) ~~\text{ where } ~~f  \mapsto \left(\mathcal{R}_{\mathcal{L}}^{\alpha}f\right)_{\mathcal{L} \in \mathbb{L}^d_k}, \notag \\
        & \text{s.t.} ~~~\mathcal{R}_{\mathcal{L}}^{\alpha}f(x_i+t \cdot \theta_i) \notag \\
        &\hspace{22pt}= \int_{\mathbb{R}^d} f(y) \cdot \alpha(y,\mathcal{L})_i \cdot \delta\left(t - \left<y-x_i,\theta_i \right>\right)~dy.
    \end{align}
Here, $\alpha$ is a continuous map from $\mathbb{R}^d \times \mathbb{L}^d_k$ to $\Delta_{k-1}$ presenting the splitting mechanism, and $(x_i, \theta_i)$ indicates $i^{\text{th}}$-line in $\mathcal{L}$. Eq.~\eqref{eq:rtsl-1} can be interpreted as integrating the $i^\text{th}$ portion of $f$, given by $f\cdot \alpha_i$, over the hyperplanes $\{y \in \mathbb{R}^d ~\colon~ \left<y,\theta_i\right> = t +\left<x_i,\theta_i\right>\}$. 

\begin{remark}
    It is important to note that the partitioning process depends on both $y \in \mathbb{R}^d$ and the set of lines $\mathcal{L}$. This splitting mechanism is absent in the traditional Radon Transform and its variants, as seen in Eqs.~\eqref{eq:traditional-Radon-Transform}, \eqref{eq:GRT}, \eqref{eq:SRT}. This naturally leads to the question: 
    
    \emph{Can a nonlinear projectional framework, similar to GRT and SRT, be developed for RTSL?} 
    
    Notably, the presence of the map $\alpha$ introduces a trade-off between the effectiveness of the induced Wasserstein metric and the associated theoretical guarantees. Empirical findings indicate that, given the same number of projections (and consequently the same computational cost), the distances derived from RTSL surpass those obtained from the original Radon Transform. However, incorporating $\alpha$ shifts the transformation from operating on straight lines to a more intricate space. This transition may compromise fundamental properties of the transform, such as injectivity, which might no longer be assured in this new framework. 
\end{remark}

In the next sections, we propose an approach for incorporating the nonlinear projectional framework into the systems of lines setting, thereby generalizing the current RTSL and inducing new variants of the tree-sliced distance.

\section{Radon Transform on Systems of Lines with Nonlinear Projection}
\label{sec:Radon Transform on Systems of Lines with Nonlinear Projection}
In this section, we extend the current Radon Transform on Systems of Lines \cite{tran2024tree, tran2025distancebased} by incorporating nonlinear projections. We then analyze key properties of the resulting transforms, including injectivity.
\subsection{Nonlinear Radon Transform on Systems of Lines}
\label{section:Nonlinear Radon Transform on Systems of Lines}
Given a positive integer $k$ representing the number of lines in a tree system, and a continuous splitting map function $\alpha \in \mathcal{C}(\mathbb{R}^d \times \mathbb{L}^d_k, \Delta_{k-1})$ defining the splitting mechanism. Let $\mathcal{L}$ be a system of $k$ lines in $\mathbb{L}^d_k$ and a scalar $r \ge 0$. For a function $f \in L^1(\mathbb{R}^d)$, define the function $\mathcal{CR}_{\mathcal{L},r}^{\alpha}f \in L^1(\mathcal{L})$ as follows:
    \begin{align} \label{eq:circular-radon-transform-SL}
    &~~~~~~~~~~~~~~\mathcal{CR}_{\mathcal{L},r}^{\alpha}f(x_i+t \cdot \theta_i) \notag \\
        &= \int_{\mathbb{R}^d} f(y) \cdot \alpha(y,\mathcal{L})_i \cdot \delta\left(t - \|y-x_i-r\theta_i\|_2\right)~dy.
    \end{align}
The \textit{Circular Radon Transform on Systems of Lines} (CRTSL) is defined as the operator:
\begin{align}
        \mathcal{CR}^{\alpha} ~ \colon~~~~~ L^1(\mathbb{R}^d) ~~~ &\longrightarrow ~\prod_{\mathcal{L} \in \mathbb{L}^d_k, r \ge 0} L^1(\mathcal{L}) \notag \\f  ~~~~~~~~~~~&\longmapsto ~ \left(\mathcal{CR}_{\mathcal{L},r}^{\alpha}f\right)_{\mathcal{L} \in \mathbb{L}^d_k, r \ge 0}.
\end{align}
This is analogous to the CRT described in Section~\ref{subsection:Circular and Spatial Radon Transforms}. In the case of SRT, consider a positive integer $d_\theta$, an injective continuous map $h \colon \mathbb{R}^d \rightarrow \mathbb{R}^{d_\theta}$, and a splitting map $\alpha \in \mathcal{C}(\mathbb{R}^{d_\theta} \times \mathbb{L}^{d_\theta}_k, \Delta_{k-1})$. Let $\mathcal{L}$ be a system of lines in $\mathbb{L}^{d_\theta}_k$. For a function $f \in L^1(\mathbb{R}^d)$, define the function $\mathcal{H}_{\mathcal{L}}^{\alpha}$ as:
    \begin{align} \label{eq:spatial-radon-transform-SL}
         &~~~~~~~~~~~~~\mathcal{H}_{\mathcal{L}}^{\alpha}f(x_i+t \cdot \theta_i)  \\
        &=\int_{\mathbb{R}^d} f(y) \cdot \alpha(h(y),\mathcal{L})_i \cdot \delta\left(t - \left<h(y)-x_i,\theta_i \right>\right)~dy. \notag
    \end{align}
The \textit{Spatial Radon Transform on Systems of Lines} (SRTSL) is defined as the operator:
    \begin{align} 
        \mathcal{H}^{\alpha} ~\colon ~~~~~~ L^1(\mathbb{R}^d) ~~~&\longrightarrow ~~ \prod_{\mathcal{L} \in \mathbb{L}^{d_\theta}_k} L^1(\mathcal{L}) ~ \notag\\f  ~~~~~~~~&\longmapsto~~~ \left(\mathcal{H}_{\mathcal{L}}^{\alpha}f\right)_{\mathcal{L} \in \mathbb{L}^{d_\theta}_k}.
    \end{align}
    
\begin{remark}
    It is important to note that when the system of lines consists of a single line, i.e. $k=1$, $\mathcal{CR}^\alpha$ and $\mathcal{H}^\alpha$ recover the GRT and the SRT, respectively.
\end{remark}
Properties of $\mathcal{CR}^{\alpha}$ and $\mathcal{H}^{\alpha}$ are discussed in the next part.

\subsection{Well-definedness and Injectivity}
\label{section:Well-definedness and Injectivity}
\textbf{Well-definedness.} Given the setting of $\mathcal{CR}^\alpha$ and $\mathcal{H}^\alpha$ as defined in Eqs.~\eqref{eq:circular-radon-transform-SL}, \eqref{eq:spatial-radon-transform-SL}, we have $\mathcal{CR}_{\mathcal{L},r}^{\alpha}f \in L^1(\mathcal{L})$ and $\mathcal{H}_{\mathcal{L}}^{\alpha}f \in L^1(\mathcal{L})$ for $f \in L^1(\mathbb{R}^d)$. Furthermore, we have the bounds:  
\begin{align}
    \|\mathcal{CR}_{\mathcal{L},r}^{\alpha}f\|_{\mathcal{L}}\le \|f\|_1 ~~~ \text{ and } ~~~ \|\mathcal{H}_{\mathcal{L}}^{\alpha}f\|_\mathcal{L}  \le \|f\|_1.
\end{align}
The proofs for these properties are provided in Appendices~\ref{appendix:well-definedness-CR} and \ref{appendix:well-definedness-S}. Additionally, these proofs imply that if $f \in \mathcal{P}(\mathbb{R}^d)$, then $\mathcal{CR}_{\mathcal{L},r}^{\alpha}f,  \mathcal{H}_{\mathcal{L}}^{\alpha}f \in \mathcal{P}(\mathcal{L})$. 

\paragraph{Injectivity.} For injectivity of $\mathcal{CR}^\alpha$ and $\mathcal{H}^\alpha$, we refer to the concept of $\operatorname{E}(d)$-invariance in splitting maps as introduced in \cite{tran2025distancebased}. The group $\operatorname{E}(d)$ represents the Euclidean group, which consists of all transformations of Euclidean space $\mathbb{R}^d$ that preserve the Euclidean distance between any two points. Through the canonical action of $\operatorname{E}(d)$ on $\mathbb{R}^d$, an induced action of $\operatorname{E}(d)$ on $\mathbb{L}^d_k$ follows. Appendix~\ref{appendix:sec:background-tsw-sl} provides a formal description of the underlying group actions associated with these equivariant constructions. A splitting map $\alpha \in \mathcal{C}(\mathbb{R}^d \times \mathbb{L}^d_k, \Delta_{k-1})$ is $\operatorname{E}(d)$-invariant if:
\begin{align}
        \alpha(gy,g\mathcal{L}) = \alpha(y,\mathcal{L}),
\end{align}
for all $(y, \mathcal{L}) \in \mathbb{R}^d \times \mathbb{L}^d_k$ and $g \in \operatorname{E}(d)$.
We have two results about injectivity of operator $\mathcal{CR}^\alpha$ and $\mathcal{H}^\alpha$.
\begin{theorem}\label{mainbody:injectivity-of-Radon-Transform-circular}
    For an $\operatorname{E}(d)-$invariant splitting map $\alpha \in \mathcal{C}(\mathbb{R}^d \times \mathbb{L}^d_k, \Delta_{k-1})$, $\mathcal{CR}^{\alpha}$ is injective.
\end{theorem}
\begin{theorem}\label{mainbody:injectivity-of-Radon-Transform-spatial}
    For an $\operatorname{E}(d_\theta)-$invariant splitting map $\alpha \in \mathcal{C}(\mathbb{R}^{d_\theta} \times \mathbb{L}^{d_\theta}_k, \Delta_{k-1})$, $\mathcal{H}^{\alpha}$ is injective.
\end{theorem}
The proofs of Theorem~\ref{mainbody:injectivity-of-Radon-Transform-circular} and Theorem~\ref{mainbody:injectivity-of-Radon-Transform-spatial} are provided in Appendices~\ref{appendix:Proof for Theorem {mainbody:injectivity-of-Radon-Transform-circular}} and \ref{Proof for Theorem {mainbody:injectivity-of-Radon-Transform-spatial}}. Intuitively, the reason why 
$\operatorname{E}(d)-$invariance in splitting maps ensures the injectivity of $\mathcal{CR}^{\alpha}$ stems from the fact that the circular defining function in Eq.~\eqref{eq:circular-defining-f} is primarily based on the Euclidean norm $\|\cdot \|_2$, which itself is an $\operatorname{E}(d)-$ function. Similarly, the reason why $\operatorname{E}(d_\theta)$-invariance in splitting maps guarantees the injectivity of $\mathcal{H}^{\alpha}$ is that this property is essential for achieving injectivity in the standard Radon Transform on Systems of Lines in $\mathbb{R}^{d_\theta}$, as discussed in \cite{tran2025distancebased}.

\subsection{Spatial Spherical Radon Transform on Spherical Trees}
In \cite{tran2025spherical}, the tree-sliced framework is extended to functions defined on hyperspheres. The techniques presented in Sections~\ref{section:Nonlinear Radon Transform on Systems of Lines} and \ref{section:Well-definedness and Injectivity} can be adapted to the spherical setting. We provide a brief derivation here, while a more detailed background and notation are given in Appendix~\ref{appendix:background-stsw}. For $f \in L^1(\mathbb{S}^d)$, we recall the Spherical Radon Transform on Spherical Trees (SRTST), which  transforms $f$ to $\mathcal{R}^\alpha_{\mathcal{T}}f \in L^1(\mathcal{T})$, where:
\begin{align}
&~~~~~~~~~~\mathcal{R}^\alpha_{\mathcal{T}}f(t,r^x_{y_i}) \notag \\ &= \int_{\mathbb{S}^d} f(y) \cdot \alpha(y, \mathcal{T})_i \cdot \delta(t  - \operatorname{arccos}\left<x,y \right>) ~ dy,
\end{align}
Since the literature on defining functions for GRT on the sphere is limited, we focus only on the spatial version of SRTST. Given a positive integer $d_\theta$ and an injective continuous map $h \colon \mathbb{S}^{d} \rightarrow \mathbb{S}^{d_\theta}$, the \textit{Spatial Spherical Radon Transform on Spherical Trees} transforms $f \in L^1(\mathbb{S}^d)$ to $\mathcal{H}^\alpha_{\mathcal{T}}f \in L^1(\mathcal{T})$, where:
\begin{align}
&~~~~~\mathcal{H}^\alpha_{\mathcal{T}}f(t,r^x_{y_i})  \\ &= \int_{\mathbb{S}^d} f(y) \cdot \alpha(h(y), \mathcal{T})_i \cdot \delta(t  - \operatorname{arccos}\left<x,h(y) \right>) ~ dy.\notag
\end{align}
As stated in \cite{tran2025spherical}, $\operatorname{O}(d+1)$-invariance is a necessary property of the splitting map $\alpha$ to ensure the injectivity of $\mathcal{R}^\alpha$. A similar result holds in our setting, as described below.
\begin{theorem}\label{mainbody:injectivity-of-generalized-spherical-radon-transform-tree}
    For an $\operatorname{O}(d_\theta+1)$-invariant splitting map $\alpha \in \mathcal{C}(\mathbb{S}^{d_\theta} \times \mathbb{T}^{d_\theta}_k,\Delta_{k-1})$, $\mathcal{H}^\alpha$ is injective.
\end{theorem}
The detailed derivation of the Spatial Spherical Radon Transform on Spherical Trees and the proof of Theorem~\ref{mainbody:injectivity-of-generalized-spherical-radon-transform-tree} are provided in Appendices~\ref{appendix:def-Spatial Spherical Radon Transform on Spherical Trees} and \ref{appendix:Proof for Theorem{mainbody:injectivity-of-generalized-spherical-radon-transform-tree}}.

\section{Tree-Sliced Wasserstein Distance with Nonlinear Projection}
\label{sec:Tree-Sliced Wasserstein Distance with Nonlinear Projection}

In this section, we propose new distance between measures derived from the variants of the Radon Transform introduced in Section \ref{sec:Radon Transform on Systems of Lines with Nonlinear Projection}. We also examine different choices of functions that define the nonlinear projections and explain why certain choices lead to more efficient metrics.

\subsection{Definition of Tree-Sliced Distances}
For two probability measures $\mu$ and $\nu$ with density function $f_\mu$ and $f_\nu$, and a fixed $r \ge 0$, the \textit{Circular Tree-Sliced Wasserstein Distance} (CircularTSW) between $\mu$ and $\nu$ is defined as the average Wasserstein distance on the tree-metric space $\mathcal{L}$ between the CRTSL of $f_\mu$ and $f_\nu$. Following \cite{tran2024tree, tran2025distancebased}, this averaging is taken over the space of trees $\mathbb{T}^d_k \subset \mathbb{L}^d_k$, according to a distribution $\sigma$ on $\mathbb{T}^d_k$ which arises from the tree sampling process.
\begin{definition}
The \textit{Circular Tree-Sliced Wasserstein Distance} between $\mu$ and $\nu$ in $\mathcal{P}(\mathbb{R}^d)$ is defined by:    
\begin{align}
\label{eq:CircularTSW-formula}
    &\text{CircularTSW} (\mu,\nu) \notag\\ 
    &~~~~~~~~~~~~~\coloneqq 
 \int_{\mathbb{T}^d_k} \text{W}(\mathcal{CR}^\alpha_{\mathcal{L},r} f_\mu, \mathcal{CR}^\alpha_{\mathcal{L},r} f_\nu) ~d\sigma(\mathcal{L}).
\end{align}    
\end{definition}
Similarly, given a choice of the continuous injective map $h\colon \mathbb{R}^d \rightarrow \mathbb{R}^{d_\theta}$ in SRTSL, we have the definition of the \textit{Spatial Tree-Sliced Wasserstein Distance} (SpatialTSW) between 
$\mu$ and $\nu$.
\begin{definition}
The \textit{Spatial Tree-Sliced Wasserstein Distance} between $\mu$ and $\nu$ in $\mathcal{P}(\mathbb{R}^d)$ is defined by:
\begin{align} 
\label{eq:SpatialTSW-formula} 
\hspace{-1.5em} \text{SpatialTSW} (\mu,\nu) \coloneqq \int_{\mathbb{T}^{d_\theta}_k} \hspace{-0.2em}\text{W}(\mathcal{H}^\alpha_\mathcal{L} f_\mu, \mathcal{H}^\alpha_\mathcal{L} f_\nu) ~d\sigma(\mathcal{L}).
\end{align}
\end{definition}

Both CircularTSW and SpatialTSW distances are, indeed, metrics on the space $\mathcal{P}(\mathbb{R}^d)$ of measures on $\mathbb{R}^d$.
\begin{theorem}
\label{thm:CircularTSW and SpatialTSW are metrics}
    CircularTSW and SpatialTSW are metrics on the space $\mathcal{P}(\mathbb{R}^d)$.
\end{theorem}
The proof of Theorem~\ref{thm:CircularTSW and SpatialTSW are metrics} is presented in Appendix~\ref{appendix:thm:CircularTSW and SpatialTSW are metrics}. The algorithms for CircularTSW and SpatialTSW are presented in Appendix~\ref{app:subsec:alg_TSW} (Alg.~\ref{alg:compute-tsw-sl-circular} and~\ref{alg:compute-tsw-sl-spatial} respectively).

\subsection{Components in CircularTSW and SpatialTSW}
Both CircularTSW and SpatialTSW distances depend on the choice of splitting maps $\alpha$. Additionally, CircularTSW is influenced by the parameter $r \ge 0$, whereas SpatialTSW is determined by the selection of the injective map $h$.

\paragraph{Splitting maps $\alpha$.} For both CircularTSW and SpatialTSW, we follow the construction of the splitting map $\alpha$ as proposed in \cite{tran2025distancebased}, defined as:
\begin{align}
    \alpha(x,\mathcal{L})_l = \operatorname{softmax} \Bigl( \{ d(x,\mathcal{L})_i\}_{i=1}^k \Bigr),
\end{align}
where $d(x,\mathcal{L})_i$ represents the distance between $x$ and $i^{\text{th}}$ line of $\mathcal{L}$. This choice of $\alpha$ ensures the $\operatorname{E}(d)$-invariance while also incorporating positional information between a point and the tree system. Consequently, it results in meaningful and varied mass distributions that adapt to each specific system. Moreover, the use of the $\operatorname{softmax}$ function guarantees that $\alpha$ produces a valid probability vector in the standard simplex $\Delta_{k-1}$. The splitting map for CircularTSW is further discussed in Appendix~\ref{appendix:splitting-map-circular-tsw}.

\paragraph{Radius $r$ in CircularTSW.}
\label{mainbody:radius-in-circular-tsw}
In the formulation of the original Circular Radon Transform, the radius $r$ plays a crucial role. Together with $\theta \in \mathbb{S}^{d-1}$, the term $r\theta$ spans the entire space $\mathbb{R}^d$. This ensures that the level sets defined by the defining circular function in Eq.~\eqref{eq:circular-defining-f} can represent arbitrary $(d-1)$-dimensional spheres in $\mathbb{R}^d$:
\begin{align}
    \{ y \in \mathbb{R}^d ~ \colon ~ t = \|y-r\theta\|_2 \} ~ \text{ for } ~ t \ge 0.
\end{align}
However, in the framework of Tree-Sliced methods, since the sources within tree systems already encompass the entire space $\mathbb{R}^d$, considering all values of $r \ge 0$ becomes redundant. In other words, for a fixed $r \ge 0$,  the level sets arising in the Circular Radon Transform on Systems of Lines from Eq.~\eqref{eq:circular-radon-transform-SL} can still effectively represent arbitrary $(d-1)$-dimensional spheres in $\mathbb{R}^d$:
\begin{align}
    \{ y \in \mathbb{R}^d ~ \colon ~ t = \|y- x_i -r\theta_i\|_2 \} ~ \text{ for } ~ t \ge 0.
\end{align}
This is why, in the definition of CircularTSW, we fix a specific $r \ge 0$. Furthermore, we want to examine CircularTSW$_{r=0}$, a special case of CircularTSW when $r=0$. In this scenario, by selecting the concurrent tree space as proposed in \citet{tran2025distancebased}, the practical implementation of CircularTSW becomes significantly more efficient, reducing computational complexity. In summary, transforming a measure $\mu \in \mathcal{P}(\mathbb{R}^d)$ involves projecting its support onto $k$ lines in the tree system. When $r=0$, all support points of $\mu$ share the same coordinates when projected onto these $k$ lines. As a result, projecting and sorting cost is reduced. Figure~\ref{fig:circular-concentric} illustrates this phenomenon. Furthermore, we emphasize that CircularTSW$_{r=0}$ is specifically designed for tree settings, where splitting maps play a crucial role. Since the coordinates are identical across these lines, the tree structure and the distance-based splitting map are necessary to distinguish between the $k$ lines. Empirical results in Appendix~\ref{appendix:ablate-k-circular-concentric} show that CircularTSW$_{r=0}$ performs well in a tree setting but poortly in original sliced setting. 

\begin{figure}[t]
\vskip 0.2in
\begin{center}
\centerline{\includegraphics[width=.7\columnwidth]{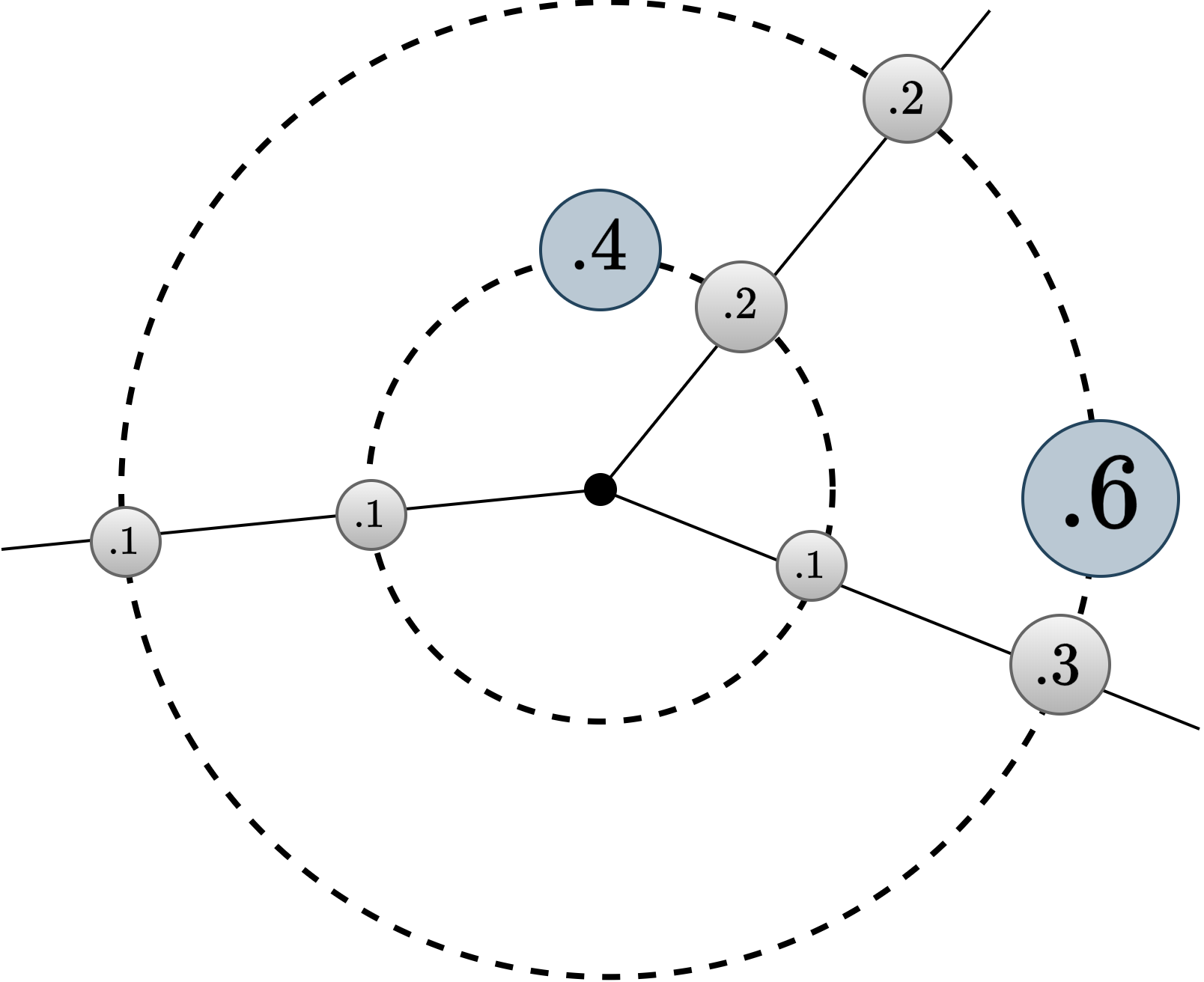}}
\caption{An illustration depicting CircularTSW$_{r=0}$. Given a support, the projection coordinates are identical when projected onto $k$ lines.}
\label{fig:circular-concentric}
\end{center}
\vspace{-25pt}
\end{figure}

\paragraph{The choice of the map $h$ in SpatialTSW.} As discussed in Section \ref{section:Nonlinear Radon Transform on Systems of Lines}, the map 
$h$ in SpatialTSW must be both injective and continuous. One approach to selecting this map is based on odd degree homogeneous polynomials. However, following the constructions in \cite{kolouri2019generalized, rouvire-2015} can lead to an excessively large new dimension $d_\theta = \binom{m+d-1}{d-1}$. To address this, we propose an alternative approach: Consider $h ~ \colon ~ \mathbb{R}^d \rightarrow \mathbb{R}^d$ defined by:
\begin{align}
    h(x_1, \ldots, x_d) = (f_1(x_1), \ldots, f_d(x_d)), 
\end{align}
where $\{f_i\colon \mathbb{R} \rightarrow \mathbb{R}\}_{i=1}^d$ are injective and continuous functions. A simple choice for $f_i$ is an odd-degree polynomial that remains injective, such as $f_i(x) = x_i + x_i^3$. Another approach, inspired by \cite{chen2022augmented}, involves concatenating the input with the output of a neural network by concatenating input with an arbitrary neural network $\phi(\cdot)$. Specifically, we define $h \colon \mathbb{R}^d \rightarrow \mathbb{R}^{d+d'}$ as $h(x) = (x,\phi(x))$, where $\phi \colon \mathbb{R}^d \rightarrow \mathbb{R}^{d'}$ is a neural network. This choice introduces learnable parameters for $h$, offering a trade-off between potentially improving performance and increasing computational cost.

\subsection{Computational Complexity}

Consider two discrete measures $\mu, \nu \in \mathcal{P}(\mathbb{R}^d)$ with $m,n$ support points, respectively. The computational complexity of the original Sliced Wasserstein distance is $\mathcal{O}(Ln \log n + Ldn)$, where $L$ represents the number of samples used in the Monte Carlo approximation \cite{peyre2019computational, nguyen2024quasimonte}. In comparison, the computational complexity of the Tree-Sliced Wasserstein (TSW) distances, such as TSW-SL in \cite{tran2024tree} or Db-TSW in \cite{tran2025distancebased}, is $\mathcal{O}(Lkn \log n + Lkdn)$, where $L$ denotes the number of tree samples used in the Monte Carlo approximation, and $k$ is the number of lines per tree. This computational difference highlights why a fair comparison between the sliced method and the tree-sliced method requires ensuring that the total number of directions remains the same. For SpatialTSW, its computational complexity is $\mathcal{O}(Lkn \log n + Lkd_\theta n)$, with an additional initial cost for computing the function $h$. This complexity matches TSW-SL and Db-TSW but operates in the transformed space $\mathbb{R}^{d_\theta}$. For CircularTSW with a general $r \geq 0$, the complexity remains $\mathcal{O}(Lkn \log n + Lkd_\theta n)$, yet it achieves faster empirical runtime than Db-TSW by computing vector norms instead of vector products. Notably, for $r=0$, CircularTSW$_{r=0}$ improves complexity to $\mathcal{O}(Ln \log n + Lkd_\theta n)$. This reduction arises because the $\mathcal{O}(n \log n)$ sorting step per line is required for only a single line, rather than all $k$ lines in a tree, as illustrated in Figure~\ref{fig:circular-concentric}. The empirical efficiency of CircularTSW and CircularTSW$_{r=0}$ is demonstrated in Figure~\ref{fig:compare-runtime}, where we use $L = 10000$ for SW and $L = 2500, k = 4$ for Tree-Sliced methods, following the practical setting used in the Diffusion Model experiment. CircularTSW$_{r=0}$ scales efficiently with the number of supports $n$ and is the only Tree-Sliced method that closely matches the speed of vanilla SW.

\subsection{Spatial Spherical Tree-Sliced Wasserstein Distance}
From the Spatial Spherical Radon Transform on Spherical Trees, we introduce a spherical variant of the Spatial Tree-Sliced Wasserstein distance, referred to as SpatialSTSW distance. A detailed derivation, along with the selection of the corresponding injective map and theoretical proofs for the SpatialSTSW distance, is provided in Appendix~\ref{appendix:dis-Spatial Spherical Tree-Sliced Wasserstein distance}.

\begin{figure}[t]
\vskip 0.2in
\begin{center}
\centerline{\includegraphics[width=\columnwidth]{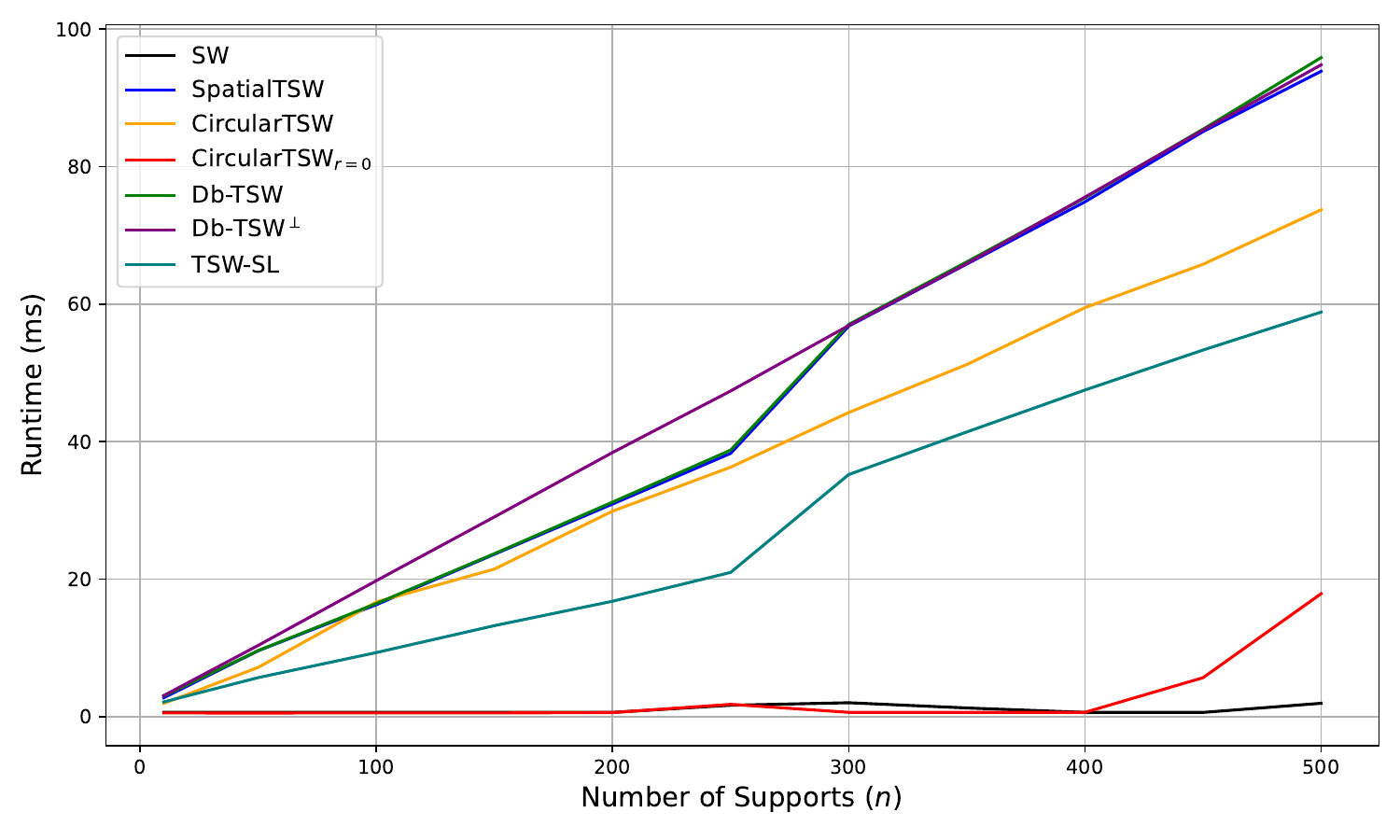}}
\caption{Runtime comparison of our proposed distances and baseline methods. We randomize the measures and projection directions and benchmark runtime over 10 runs. CircularTSW$_{r=0}$ is comparable to SW and significantly outperforms existing Tree-Sliced distances in terms of speed.}
\label{fig:compare-runtime}
\end{center}
\vspace{-25pt}
\end{figure}

%% file: sections/experiment.tex
\section{Experimental Results}
\label{sec:Experimental Results}
In this section, we thoroughly assess the validity of our proposed metric through extensive numerical experiments on both Euclidean and spherical datasets.
\subsection{Euclidean Datasets}

\paragraph{Denoising Diffusion Generative Adversarial Network.} This experiment investigates training denoising diffusion models for unconditional image synthesis. Following the approach of \citet{nguyen2024sliced}, we incorporate a Wasserstein distance into the Augmented Generalized Mini-batch Energy (AGME) loss function of the Denoising Diffusion Generative Adversarial Network (DDGAN) \citep{xiao2021tackling}. We benchmark our proposed methods -- SpatialTSW-DD, CircularTSW, and CircularTSW$_{r=0}$ -- against Sliced and Tree-Sliced Wasserstein-based DDGAN variants, as detailed in Table~\ref{table:diffusion}. All models are trained for $1800$ epochs on the CIFAR10 dataset \citep{krizhevsky2009learning}. For vanilla SW and its variants, we follow the parameter settings from \citet{nguyen2024sliced}, using $L = 10000$. For Tree-Sliced methods, including our own, we adopt the configuration from \citet{tran2025distancebased}, setting $L = 2500$ and $k = 4$.  Further details on this experiment can be found in Appendix~\ref{appendix:exp-ddgan}

The results in Table~\ref{table:diffusion} show that SpatialTSW-DD, CircularTSW-DD, and CircularTSW$_{r=0}$-DD achieve notable improvements in FID compared to all baselines. They surpass the current state-of-the-art OT-based DDGAN, Db-TSW-DD$^\perp$ \citep{tran2025distancebased}, by margins of $0.01$, $0.2$, and $0.05$, respectively. Additionally, our methods offer faster training times compared to existing Tree-Sliced approaches. CircularTSW-DD and CircularTSW reduce training time relative to Db-TSW-DD$^\perp$ by 10\% and 19\%, respectively. These enhancements in both training efficiency and model performance underscore the practical advantages of our proposed methods.
\input{tables/diffusion}

\paragraph{Gradient Flow.} The goal of gradient flow is to minimize the distance between a source distribution $\mu$ and a target distribution $\nu$ through gradient-based optimization. The update rule follows $ \partial_t \mu_t = -\nabla \mathcal{D}(\mu_t,\nu), \mu_0 = \mathcal{N}(0, 1),$ where $\mu_t$ represents the distribution at time $t$, and $\nabla \mathcal{D}(\mu_t,\nu)$ is the gradient of the distance function $\mathcal{D}$ with respect to $\mu_t$. We evaluate SpatialTSW, CircularTSW, CircularTSW$_{r=0}$, and several established Sliced-Wasserstein (SW) variants, including vanilla SW \citep{bonneel2015sliced}, MaxSW \citep{deshpande2019max}, LCVSW \citep{nguyen2023sliced}, SWGG \citep{mahey2023fast}, alongside the recently introduced Tree-Sliced distances, such as TSW-SL \citep{tran2024tree}, Db-TSW$^\perp$, and Db-TSW \citep{tran2025distancebased}. We conduct experiments on the \textit{25 Gaussians} dataset and use the Wasserstein distance to evaluate the average distance between the source and target distributions. We report results over 5 runs at iterations $500$, $1000$, $1500$, $2000$, and $2500$.

The results presented in Table~\ref{table:gradientflow} demonstrate that SpatialTSW achieves the best performance across all iterations, reaching a final $W_2$ distance of $1.17\mathrm{e}{-7}$ at the last step. This represents a significant improvement over vanilla SW ($3.59\mathrm{e}{-2}$) and LCVSW ($9.28\mathrm{e}{-3}$). Furthermore, compared to the best existing Tree-Sliced distances, SpatialTSW achieves better results than Db-TSW ($1.3\mathrm{e}{-7}$), exhibiting faster convergence, and maintaining similar computational efficiency. In this experiment, SpatialTSW outperforms CircularTSW, aligning with the findings from \citep{kolouri2019generalized}, where polynomial-based defining functions yield superior results over circular defining functions in gradient flow tasks. Nevertheless, both CircularTSW and CircularTSW$_{r=0}$ still achieve better results than vanilla SW while offering significant computational speedups. Specifically, CircularTSW and CircularTSW$_{r=0}$ are approximately 5\% and 16\% faster than vanilla SW, respectively.

\input{tables/gradient_flow}

\subsection{Spherical Datasets}

\input{sections/spherical_exp/gradient_flow}

\paragraph{Self-Supervised Learning (SSL).}
\input{sections/spherical_exp/ssl}


\paragraph{Sliced-Wasserstein Auto-Encoder.}
\input{sections/spherical_exp/swae}

%% file: tables/diffusion.tex
\begin{table}[t]
\caption{Fréchet Inception Distance (FID) scores and per-epoch training times of different DDGAN variants for unconditional generation on CIFAR-10.}
\vspace{-15pt}
\label{table:diffusion}
\vskip 0.15in
\begin{center}
\begin{small}
\begin{adjustbox}{width=.48\textwidth} 
\begin{tabular}{lcccr} 
    \toprule 
    Model & FID $\downarrow$ & Time/Epoch(s) $\downarrow$  \\ 
    \midrule 
    DDGAN \cite{xiao2021tackling} & 3.64 & 188 \\
    SW-DD \cite{nguyen2024sliced} & 2.90 & 192  \\
    DSW-DD \cite{nguyen2024sliced} & 2.88 & 1268 \\
    EBSW-DD \cite{nguyen2024sliced} & 2.87 & 188 \\
    RPSW-DD \cite{nguyen2024sliced} & 2.82 & 194  \\
    IWRPSW-DD \cite{nguyen2024sliced} & 2.70 & 194  \\
    \midrule
    TSW-SL-DD \cite{tran2024tree} & 2.83 & 249 \\
    Db-TSW-DD \cite{tran2025distancebased} & 2.60 & 256 \\
    Db-TSW-DD$^\perp$ \cite{tran2025distancebased} & 2.53 & 262 \\
    \midrule
    SpatialTSW-DD (ours) & 2.52 & 262 \\
    CircularTSW-DD (ours) & \textbf{2.33} & 234 \\
    CircularTSW$_{r=0}$-DD (ours) & \underline{2.48} & 211 \\
    \bottomrule 
\end{tabular}
\end{adjustbox}
\end{small}
\end{center}
\vspace{-20pt}
\end{table}

%% file: tables/gradient_flow.tex
\begin{table}[t]
\caption{Average Wasserstein distance between source and target distributions of $5$ runs on 25 Gaussians datasets. All methods use $100$ projecting directions.}
\vspace{-15pt}
\label{table:gradientflow}
\vskip 0.15in
\begin{center}
\begin{small}
\begin{adjustbox}{width=0.48\textwidth}
\begin{tabular}{lcccccccc}
\toprule
\multirow{2}{*}{Methods}       & \multicolumn{5}{c}{Iteration}                                                                     & \multirow{2}{*}{Time/Iter($s$)} \\ 
\cmidrule(lr){2-6}
                                & 500              & 1000             & 1500             & 2000             & 2500             &           \\ \midrule
SW                              & 4.21e-1 & 1.54e-1          & 7.72e-2          & 4.97e-2          & 3.59e-2          & 0.0018     \\ 
MaxSW                           & 5.23e-1          & 2.36e-1          & 1.23e-1          & 8.04e-2          & 6.76e-2          & 0.1020      \\  
SWGG                            & 6.59e-1 & 3.62e-1 & 1.92e-1          & 9.07e-2          & 4.42e-2          & 0.0019    \\ 
LCVSW                           & \underline{3.46e-1}          & \underline{6.96e-2}          & 2.26e-2          & 1.31e-2          & 9.28e-3          & 0.0019    \\ 
\midrule
TSW-SL                          & 3.49e-1          & 8.10e-2          & \underline{1.06e-2}          & 2.68e-3          & 3.16e-6          & 0.0019     \\ 
Db-TSW                       & 3.50e-1          & 8.12e-2          & 1.09e-2 & \underline{1.77e-3} & \underline{1.30e-7} & 0.0020     \\
Db-TSW$^\perp$               & 3.52e-1          & 7.69e-2 & 2.73e-2 & 2.56e-3 & 2.03e-6 & 0.0021     \\
\midrule
SpatialTSW                          & \textbf{3.20e-1}          & \textbf{3.44e-2}          & \textbf{2.95e-3}          & \textbf{3.97e-4}          & \textbf{1.17e-7}          & 0.0021     \\ 
CircularTSW               & 4.28e-1          & 1.20e-1 & 3.48e-2 & 1.41e-2 &  7.86e-3 & 0.0017     \\
CircularTSW$_{r=0}$                       & 4.32e-1          & 1.22e-1          & 3.41e-2 & 1.45e-2 & 8.94e-3 & 0.0015     \\
\bottomrule
\end{tabular}
\end{adjustbox}
\end{small}
\end{center}
\vspace{-15pt}
\end{table}

%% file: sections/spherical_exp/gradient_flow.tex
\paragraph{Gradient Flow on The Sphere.} We now evaluate the ability to learn distributions by iteratively minimizing $d(\nu, \mu)$, where $d$ is a distance metric such as SSW \cite{bonet2022spherical}, S3W variants \cite{tran2024stereographic} and STSW \cite{tran2025spherical}. In line with previous works \cite{tran2024stereographic, tran2025spherical}, we consider a mixture of 12 von Mises-Fisher distributions (vMFs) from which we have access to a sample set $\{y_i\}_{i=1}^M$ with $M=2400$. The optimization procedure uses projected gradient descent \cite{bonet2022spherical}, applied on the sphere with full-batch size. Table \ref{table:sphGF} illustrates the evolution of the log $2$-Wasserstein distance at epochs 50, 100, 150, 200, and 250, averaged over 5 runs. From these results, SpatialSTSW demonstrates better performance compared to baselines while maintaining computational efficiency close to STSW.

\begin{table}[t]
\caption{Average Log of the Wasserstein distance between source and target distributions over $5$ runs on a mixture of 12 vMFs.}
\label{table:sphGF}
\vskip 0.15in
\begin{center}
\begin{small}
\vspace{-15pt}
\begin{adjustbox}{width=0.48\textwidth}
\begin{tabular}{lcccccccc}
\toprule
\multirow{2}{*}{Methods} & \multicolumn{5}{c}{Epoch} & \multirow{2}{*}{Time/Epoch($s$)} \\ 
\cmidrule(lr){2-6}
 & 50 & 100 & 150 & 200 & 250 & \\ 
\midrule
SSW & -2.4274 & -2.7893 & -2.9226 & -2.9882 &-3.0313 & 0.4323 \\
S3W & -2.0204 & -2.1920 & -2.2615 & -2.2699 & -2.2734 & 0.0151 \\
RI-S3W (1) & -2.1107 & -2.5163 & -2.7295 & -2.8568 & -2.9447 & 0.0182 \\
RI-S3W (5) & -2.4399 & -2.8273 & -3.0093 & -3.1234 & -3.2145 & 0.0503 \\
ARI-S3W (30) & -2.6508 & -3.0279 & -3.2405 & -3.4385 & -3.6661 & 0.1884\\
\midrule
STSW & \textbf{-2.9545} & \textbf{-3.5322} & \underline{-3.9992} & \underline{-4.3623} & \underline{-4.6486} & \textbf{0.0134} \\
SpatialSTSW & \underline{-2.8824} & \underline{-3.4626} & \textbf{-4.0903} & \textbf{-4.5368} & \textbf{-4.6859} & \underline{0.0145} \\

\bottomrule
\end{tabular}
\end{adjustbox}
\end{small}
\end{center}
\vskip -0.03in
\end{table}

%% file: sections/spherical_exp/ssl.tex
In earlier work, \citet{wang2020understanding} have shown that the contrastive objective can be broken down into an alignment loss, which ensures that representations of similar inputs are close together, and a uniformity loss, which prevents representations from collapsing by encouraging them to spread out evenly. Inspired by \citet{bonet2022spherical}, we replace the Gaussian kernel in uniformity loss with our SpatialSTSW:
\begin{align*}
    \mathcal{L} &= \underbrace{\frac{1}{n} \sum_{i=1}^{n} \left\| z_i^A - z_i^B \right\|_2^2}_{\text{Alignment loss}} + \\
    & \frac{\lambda}{2} \underbrace{\left( \text{SpatialSTSW}(z^A, \nu) + \text{SpatialSTSW}(z^B, \nu) \right)}_{\text{Uniformity loss}},
\label{ssl:lossEq}
\end{align*}
where $\nu = \mathcal{U}(\mathbb{S}^{d})$ is the uniform distribution on $\mathbb{S}^d$, $z^{A}, z^{B} \in \mathbb{R}^{n\times (d + 1)}$ are feature embeddings of two augmented views of the same image and $\lambda > 0$ is regularization factor that balances the loss components. Following a similar approach to \citet{bonet2022spherical, tran2024stereographic, tran2025spherical}, we use the above objective function to pretrain a ResNet18 \cite{he2016deep} encoder on CIFAR-10 \cite{krizhevsky2009learning} for 200 epochs and then train a linear classifier to evaluate learned features. The results in Table \ref{table:ssl} indicate that SpatialSTSW achieves the best performance compared to various baseline methods, including Hypersphere \citep{wang2020understanding}, SimCLR \citep{chen2020simple}, SSW, S3W variants, and STSW.

\begin{table}[t]
\caption{Accuracy of the linear classifier on encoded (E) features and projected (P) features on $\mathbb{S}^9$. ARI-S3W and RI-S3W use 5 rotations.}
\vspace{-15pt}
\label{table:ssl}
\vskip 0.15in
\begin{center}
\begin{small}
\begin{adjustbox}{width=0.48\textwidth} 
\begin{tabular}{lccc}
\toprule
Method & Acc. E(\%) $\uparrow$ & Acc. P(\%) $\uparrow$ & Time (s/ep.)\\
\midrule
Hypersphere & 79.78 & 74.60 & 13.10 \\
SimCLR & 79.86 & 72.79 & \textbf{12.71} \\
\midrule
SSW & 70.37 & 64.76 & 13.31 \\
S3W & 78.53 & 73.73 & 12.90 \\
RI-S3W (5) & 79.96 & 74.02 & 13.08 \\
ARI-S3W (5) & 80.06 & 75.10 & 13.01 \\
\midrule
STSW & 80.51 & 76.79 & \underline{12.81} \\
SpatialSTSW & \textbf{80.68} & \textbf{77.31} & 12.87 \\
\bottomrule
\end{tabular}
\end{adjustbox}
\end{small}
\end{center}
\vskip -0.1in
\end{table}

%% file: sections/spherical_exp/swae.tex
In this study, we utilize the Sliced-Wasserstein Auto-Encoder (SWAE) framework introduced by \citet{kolouri2018sliced} to evaluate the performance of various distances, including SW, SSW \citep{bonet2022spherical}, S3W variants \citep{tran2024stereographic}, and STSW \citep{tran2025spherical}. The target of SWAE is to ensure that the encoded embeddings follow a predefined prior distribution $q$ in the latent space. Let the encoder be denoted as $\varphi: \mathcal{X} \rightarrow \mathbb{S}^{d}$ and the decoder as $\psi: \mathbb{S}^{d} \rightarrow \mathcal{X}$. The optimizing objective is defined as:
\begin{equation*}
    \min_{\varphi, \psi} \mathbb{E}_{x \sim p}\left[ c(x, \psi(\varphi(x))) \right] + \lambda\cdot\text{SpatialSTSW}\left(\varphi_{\sharp}p, q\right),
\label{swae:loss}
\end{equation*}
where $p$ represents the data distribution, $\lambda$ acts as a regularization weight, and $c(\cdot, \cdot)$ measures the reconstruction error. For reconstruction loss, we use Binary Cross Entropy (BCE) and adopt a mixture of 10 von Mises-Fisher (vMF) distributions as the prior. We report results in Table \ref{table:swae}, where we evaluate performance using the same metrics as in \citet{tran2024stereographic, tran2025spherical}. We observe that SpatialSTSW achieves better results in terms of $\log W_2$ and NLL while maintaining a competitive reconstruction loss (BCE) and efficient computation times.

\begin{table}[t]
\caption{CIFAR-10 results for SWAE evaluated on latent regularization.}
\vspace{-15pt}
\label{table:swae}
\vskip 0.15in
\begin{center}
\begin{small}
\begin{adjustbox}{width=0.48\textwidth} 
\begin{tabular}{lcccc} 
\toprule
Method & log $W_2$ $\downarrow$ & NLL $\downarrow$& BCE $\downarrow$ & Time (s/ep.)\\
\midrule
SW & -3.3181 & -0.0010 & 0.6330 & \textbf{6.8939} \\
SSW & -2.3425 & 0.0037 & 0.6316 & 16.8639\\
\midrule
S3W & -3.3181 & 0.0018 & \textbf{0.6307} & 8.9703 \\
RI-S3W & -3.1857 & -0.0034 & 0.6357 & 10.1904 \\
ARI-S3W & -3.3850 & 0.0020 & 0.6328 & 9.5882 \\
\midrule
STSW & -3.4098 & -0.0045 & 0.6347 & 7.1623 \\
SpatialSTSW & \textbf{-3.4254} & \textbf{-0.0049} & 0.6368 & 7.2811\\
\bottomrule
\end{tabular}
\end{adjustbox}
\end{small}
\end{center}
\vspace{-10pt}
\end{table}

%% file: sections/conclusion.tex
\section{Conclusion}
This paper introduces the Circular Tree-Sliced Wasserstein Distance (CircularTSW) and the Spatial Tree-Sliced Wasserstein Distance (SpatialTSW) as novel approaches for comparing probability measures in Euclidean spaces. These approaches integrate nonlinear projection techniques from the Sliced Wasserstein distance into the recent Tree-Sliced Wasserstein framework, resulting in enhanced performance and more efficient metric computations. The paper presents a formal derivation of these metrics and provides comprehensive theoretical guarantees to ensure their practical applicability. Furthermore, the proposed techniques are extended to the tree-sliced framework with a spherical setting. Experimental evaluations show that CircularTSW, SpatialTSW, and their spherical variant consistently outperform state-of-the-art Sliced Wasserstein and Tree-Sliced Wasserstein methods across various tasks, including gradient flows and diffusion models, while maintaining comparable or improved runtime efficiency. These results highlight Tree-Sliced Wasserstein distance as a promising and impactful research direction, complementing the Sliced Wasserstein distance in practical applications.

%% file: sections/appendix.tex
\begin{center}
{\bf \Large{Supplemental Material for \\ ``Tree-Sliced Wasserstein Distance with Nonlinear Projection"}}
\end{center}

The supplementary is organized into four parts as follows:
\begin{itemize}
\item In Section~\ref{appendix:sec:background-tsw-sl}, we provide background for Tree-Sliced Wasserstein distance.

\item In Section~\ref{appendix:Theoretical Proofs}, we derive theoretical proofs for Radon transform on systems of lines with nonlinear projection.

\item In Section~\ref{app:sec:radon_sphere}, we describe Radon transform on systems of lines for spherical functions.

\item In Section~\ref{app:sec:experiment_details}, we provide further details for the experiments.

\end{itemize}

\section{Background for Tree-Sliced Wasserstein Distance}
\label{appendix:sec:background-tsw-sl}

In this section, we briefly outline the notion of the Radon Transform on Systems of Lines \cite{tran2024tree} with its distance-based extension \cite{tran2025distancebased}.

\textbf{Building blocks of Tree-sliced Wasserstein distance on Systems of Lines. } The Tree-sliced Wasserstein distance on Systems of Lines is constructed step-by-step as follows:
\begin{enumerate}
    \item Given a positive number $d$ presenting the dimension.
    \item  A \textit{line} in $\mathbb{R}^d$ is an element $l = (x,\theta) \in \mathbb{R}^d \times \mathbb{S}^{d-1}$. Here, $x$ is called the \textit{source} and $\theta$ is called the \textit{direction} of the line.  
    \item A \textit{system of $k$ lines} in $\mathbb{R}^d$ is an element of $(\mathbb{R}^d \times \mathbb{S}^{d-1})^k$. Denote a system of lines as $\mathcal{L}$,  and the space of all systems of $k$ lines by $\mathbb{L}^d_k$. 
    \item \textit{A point $x$ in $\mathcal{L}$} can be parameterized as $x_i+t \cdot \theta_i$, where $i$ is the index of the line, and $t$ is the coordinate of the point on that $i^{\text{th}}$ lines.
    \item A system of line $\mathcal{L}$ with additional tree structure is called a \textit{tree system} (see \cite{tran2024tree, tran2025distancebased}). Each tree system is a measure space, endowed with a tree metric.
    \item A \textit{space of trees} (collections of all tree systems with the same tree structure) is denoted by $\mathbb{T}$ with a probability distribution $\sigma$ on $\mathbb{T}$, which comes from the tree sampling process.
    \item For $\mathcal{L} \in \mathbb{L}^d_k$, \textit{the space of integrable functions on $\mathcal{L}$} is:
    \begin{align}
    L^1(\mathcal{L}) = \left \{ f \colon \mathcal{L} \rightarrow \mathbb{R} ~ \colon ~ \|f\|_{\mathcal{L}} = \sum_{i=1}^k \int_{\mathbb{R}} |f(x_i+t\cdot \theta_i)|  \, dt < \infty \right \}.
\end{align}
    \item A splitting map $\alpha$ is a continuous map from $\mathbb{R}^d \times \mathbb{L}^d_k$ to the $(k-1)$-dimensional  standard simplex $\Delta_{k-1}$, i.e. $\alpha \in \mathcal{C}(\mathbb{R}^d\times \mathbb{L}^d_, \Delta_{k-1})$. For $f \in L^1(\mathbb{R}^d)$, we define:
    \begin{align}
\mathcal{R}_{\mathcal{L}}^{\alpha}f~~~ \colon ~~~~~~~~~~\mathcal{L}~~~~~~ &\longrightarrow ~~~~~~\mathbb{R} \\ x_i+t\cdot \theta_i &\longmapsto \int_{\mathbb{R}^d} f(y) \cdot \alpha(y,\mathcal{L})_i \cdot \delta\left(t - \left<y-x_i,\theta_i \right>\right)  ~ dy.
\end{align}
The function $\mathcal{R}_{\mathcal{L}}^{\alpha}f$ is in $L^1(\mathcal{L})$. 
\item The operator:
    \begin{align*}
        \mathcal{R}^{\alpha}~ \colon ~ L^1(\mathbb{R}^d)~  &\longrightarrow ~ \prod_{\mathcal{L} \in \mathbb{L}^d_k} L^1(\mathcal{L}) \\ f ~~~~ &\longmapsto ~ \left(\mathcal{R}_{\mathcal{L}}^{\alpha}f\right)_{\mathcal{L} \in \mathbb{L}^d_k}
    \end{align*}
    is called the \textit{Radon Transform on Systems of Lines}. 
    \item When the splitting map $\alpha$ is $\operatorname{E}(d)$-invariant (this $\operatorname{E}(d)$-invariance will be described in the next part), the Radon Transform on Systems of Lines is \textit{injective} (see \cite{tran2025distancebased}). 
    \item The \textit{Tree-Sliced Wasserstein Distance on Systems of Lines}, denoted by TSW-SL as in \cite{tran2024tree}, or its Distance-based variant Db-TSW as in \cite{tran2025distancebased}, between $\mu,\nu$ in $\mathcal{P}(\mathbb{R}^d)$ is defined by:
\begin{equation}
\label{eq:old-TSW-SL-formula}
    \text{Db-TSW} (\mu,\nu) =  \int_{\mathbb{T}^d_k} \text{W}_{d_\mathcal{L},1}(\mathcal{R}^\alpha_\mathcal{L} f_\mu, \mathcal{R}^\alpha_\mathcal{L} f_\nu) ~d\sigma(\mathcal{L}).
\end{equation}
    We choose the notation Db-TSW since this variant is a generalization of TSW-SL. Db-TSW is identical with the definition of SW when $k=1$, i.e. tree systems in $\mathbb{L}^d_k$ have only one line.
    \item The Db-TSW distance is a metric on $\mathcal{P}(\mathbb{R}^d)$ (see \cite{tran2025distancebased}). 
    \item It is worth noting that, on tree systems, optimal transport problems admits closed-form expression, since it is a metric space with tree metric (see \cite{le2019tree}). Leveraging this closed-form expression and the Monte Carlo method, the distance in Eq.~(\ref{eq:old-TSW-SL-formula}) can be efficiently approximated by a closed-form expression. Additionally, for the $p$-order Wasserstein with $p > 1$, one may consider the scalable variant---Sobolev transport~\citep{le2022sobolev}, which also yields a closed-form expression for a fast computation, and generalizes tree-Wasserstein (i.e., $1$-order Wasserstein on a tree) to a more general settings such as for $p>1$, and for measures on a graph.
\end{enumerate}

\textbf{The group $\operatorname{E}(d)$ and its action. } The Euclidean group E(d) is the group of all transformations of $\mathbb{R}^d$ that preserve the Euclidean
distance between any two points. It is the semidirect product between $\operatorname{T}(d)$ and $\OO(d)$, i.e.
\begin{align}
    \operatorname{E}(d) \simeq \operatorname{T}(d) \rtimes \operatorname{O}(d),
\end{align} 
where $\operatorname{T}(d)$ is group of all translations in $\mathbb{R}^d$ and $\OO(d)$ is the orthogonal group of $\mathbb{R}^d$. Each element $g$ of $\operatorname{T}(d)$ can be presented as a pair:
\begin{align}
    g = (Q,a) \in \operatorname{T}(d) ~~ \text{ where } ~~ Q \in \operatorname{O}(d) \text{ and } a \in \mathbb{R}^d.
\end{align}
The group $\operatorname{E}(d)$ acts on $\mathbb{R}^d$ naturally as follows: For $x \in \mathbb{R}^d$ and $g = (Q,a) \in \operatorname{T}(d)$, we have:
\begin{align}
    (g,x) \longmapsto gx = Q \cdot x + a.
\end{align}
It naturally induces a group action on the set of all lines in $\mathbb{R}^d$, i.e. $\mathbb{R}^d \times \mathbb{S}^{d-1}$: For $l = (x,\theta) \in \mathbb{R}^d \times \mathbb{S}^{d-1}$ and $g = (Q,a) \in \operatorname{E}(d)$, we have:
\begin{align}
    (g,l) \longmapsto gl = (Q\cdot  x+ a, Q \cdot \theta) \in \mathbb{R}^d \times \mathbb{S}^{d-1}.
\end{align}
For $\mathcal{L} = \bigl\{l_i = (x_i,\theta_i)\bigr\}_{i=1}^{k} \in \mathbb{L}^d_k$, the action of $\operatorname{E}(d)$ on $\mathbb{L}^d_k$ is defined as:
\begin{align}
    g\mathcal{L} = \bigl\{gl_i = (Q\cdot  x_i+ a, Q \cdot \theta_i)\bigr\}_{i=1}^k \in \mathbb{L}^d_k.
\end{align}
The tree structure of a tree system is preserved under the action of $\operatorname{E}(d)$ (see  \cite{tran2024tree, tran2025distancebased}). In other words, if $\mathcal{L} \in \mathbb{T}$ is a tree system, then $g\mathcal{L}$ is also a tree system. The group action of $\operatorname{E}(d)$ on $\mathbb{L}^d_k$ induces a group action of $\operatorname{E}(d)$ on $\mathbb{T}$.

\textbf{$\operatorname{E}(d)$-invariant splitting maps.} A splitting map $\alpha \in \mathcal{C}(\mathbb{R}^d \times \mathbb{L}^d_k, \Delta_{k-1})$ is $\operatorname{E}(d)$-invariant, if:
    \begin{align}
        \alpha(gy,g\mathcal{L}) = \alpha(y,\mathcal{L}),
    \end{align}
    for all $(y, \mathcal{L}) \in \mathbb{R}^d \times \mathbb{L}^d_k$ and $g \in \operatorname{E}(d)$.

\begin{remark}
Equivariance is widely used in machine learning across various contexts, including equivariant models~\cite{tran2024clifford} and equivariant metanetworks~\cite{vo2024equivariant,tran2024equivariant,tran2024monomial,tran2024equivariant}. These approaches leverage symmetries in data or model architectures to improve generalization, reduce sample complexity, and ensure consistency under group transformations. 
\end{remark}

\section{Theoretical Proofs for Radon Transform on Systems of Lines with Nonlinear Projection}
\label{appendix:Theoretical Proofs}
\subsection{$\mathcal{CR}^\alpha$ is well-defined}
\label{appendix:well-definedness-CR}
We show that the Circular Radon Transform on Systems of Lines is well-defined. 

\begin{proof}
    Recall from Eq.~\eqref{eq:circular-radon-transform-SL}, we have:
\begin{align}
        &~~~~~~\mathcal{CR}^{\alpha} \colon L^1(\mathbb{R}^d) \longrightarrow \prod_{\mathcal{L} \in \mathbb{L}^d_k,r \ge 0} L^1(\mathcal{L})~~ \text{ where }~~f  \longmapsto \left(\mathcal{CR}_{\mathcal{L},r}^{\alpha}f\right)_{\mathcal{L} \in \mathbb{L}^d_k, r\ge 0},
         \\ &\text{and}~~~\mathcal{CR}_{\mathcal{L},r }^{\alpha}f(x_i+t \cdot \theta_i) = \int_{\mathbb{R}^d} f(y) \cdot \alpha(y,\mathcal{L})_i \cdot \delta\left(t - \|y-x_i-r\theta_i\|_2\right)~dy.
    \end{align}
    We have:
    \allowdisplaybreaks
    \begin{align}
    \|\mathcal{CR}_{\mathcal{L},r}^{\alpha}f\|_{\mathcal{L}} &= \sum_{i=1}^k \int_{\mathbb{R}} \left | \mathcal{CR}_{\mathcal{L},r}^{\alpha}f(x_i + t \cdot \theta) \right |  \, dt_x \notag\\
        &= \sum_{i=1}^k \int_{\mathbb{R}} \left |\int_{\mathbb{R}^d} f(y) \cdot \alpha(y,\mathcal{L})_i \cdot \delta\left(t - \|y-x_i-r\theta_i\|_2\right)~dy \right| \, dt\notag \\
        & \le \sum_{i=1}^k \int_{\mathbb{R}} \left(\int_{\mathbb{R}^d} |f(y)| \cdot \alpha(y,\mathcal{L})_i \cdot \delta\left(t - \|y-x_i-r\theta_i\|_2\right)~dy  \right) \, dt\notag \\
        &= \sum_{i=1}^k \int_{\mathbb{R}^d} \left(\int_{\mathbb{R}} |f(y)| \cdot \alpha(y,\mathcal{L})_i \cdot \delta\left(t - \|y-x_i-r\theta_i\|_2\right)~dt  \right) \, dy\notag \\
        &= \sum_{i=1}^k \int_{\mathbb{R}^d}  |f(y)| \cdot \alpha(y,\mathcal{L})_i \cdot \left(\int_{\mathbb{R}}\delta\left(t - \|y-x_i-r\theta_i\|_2\right)~dt  \right) \, dy \notag\\
        &=\sum_{i=1}^k \int_{\mathbb{R}^d}  |f(y)| \cdot \alpha(y,\mathcal{L})_i \, dy\notag \\
        &= \int_{\mathbb{R}^d}  |f(y)| \cdot \left(\sum_{i=1}^k \alpha(y,\mathcal{L})_i \right) \, dy\notag \\
        &= \int_{\mathbb{R}^d}  |f(y)| \, dy \notag \\
        & = \|f \|_1.
    \end{align}
So $\mathcal{CR}^\alpha_{\mathcal{L},r} f \in L^1(\mathcal{L})$. It implies that the operator $\mathcal{CR}_{\mathcal{L},r}^{\alpha} \colon L^1(\mathbb{R}^{d}) \to L^1(\mathcal{L})$ is well-defined, as well as $\mathcal{CR}^{\alpha}$.
\end{proof}
\begin{remark}
    Note that, from the above proof, we see that if $f \in \mathcal{P}(\mathbb{R}^d)$, i.e. $f \in L^1(\mathbb{R}^d)$, $\|f \|_1 = 1$ and $f(y) \ge 0$ for all $ y \in \mathbb{R}^d$, we also have $\|\mathcal{CR}_{\mathcal{L},r}^{\alpha}f\|_{\mathcal{L}} =1$ and $\mathcal{CR}^\alpha_{\mathcal{L},r} f(x_i+t\cdot \theta_i) \ge 0$ for all $x_i+t\cdot \theta_i \in \mathcal{L}$. It implies that $\mathcal{CR}_{\mathcal{L},r}^{\alpha}f \in \mathcal{P}(\mathcal{L})$.
\end{remark}

\subsection{$\mathcal{H}^\alpha$ is well-defined}
\label{appendix:well-definedness-S}
We show that the Spatial Radon Transform on Systems of Lines is well-defined. 

\begin{proof}
    Recall from Eq.~\eqref{eq:spatial-radon-transform-SL}, we have:
    \begin{align}
        &~~~~~\mathcal{H}^{\alpha} \colon L^1(\mathbb{R}^d) \rightarrow \prod_{\mathcal{L} \in \mathbb{L}^{d_\theta}_k} L^1(\mathcal{L}) ~~\text{ where }~~f  \mapsto \left(\mathcal{H}_{\mathcal{L}}^{\alpha}f\right)_{\mathcal{L} \in \mathbb{L}^{d_\theta}_k}, \notag \\
        &\text{and }~~\mathcal{H}_{\mathcal{L}}^{\alpha}f(x_i+t \cdot \theta_i)  =\int_{\mathbb{R}^d} f(y) \cdot \alpha(h(y),\mathcal{L})_i \cdot \delta\left(t - \left<h(y)-x_i,\theta_i \right>\right)~dy.
    \end{align}
    We have:
    \allowdisplaybreaks
    \begin{align}
    \|\mathcal{H}_{\mathcal{L}}^{\alpha}f\|_{\mathcal{L}} &= \sum_{i=1}^k \int_{\mathbb{R}} \left | \mathcal{H}_{\mathcal{L}}^{\alpha}f(x_i+t \cdot \theta_i) \right |  \, dt_x \notag\\
        &= \sum_{i=1}^k \int_{\mathbb{R}} \left |\int_{\mathbb{R}^d} f(y) \cdot \alpha(h(y),\mathcal{L})_i \cdot \delta\left(t - \left<h(y)-x_i,\theta_i \right>\right)~dy \right| \, dt\notag \\
        & \le \sum_{i=1}^k \int_{\mathbb{R}} \left (\int_{\mathbb{R}^d} |f(y)| \cdot \alpha(h(y),\mathcal{L})_i \cdot \delta\left(t - \left<h(y)-x_i,\theta_i \right>\right)~dy \right) \, dt\notag  \\
        &= \sum_{i=1}^k \int_{\mathbb{R}^d} \left (\int_{\mathbb{R}} |f(y)| \cdot \alpha(h(y),\mathcal{L})_i \cdot \delta\left(t - \left<h(y)-x_i,\theta_i \right>\right)~dt \right) \, dy\notag  \\
        &= \sum_{i=1}^k \int_{\mathbb{R}^d}  |f(y)| \cdot \alpha(h(y),\mathcal{L})_i \cdot \left (\int_{\mathbb{R}}\delta\left(t - \left<h(y)-x_i,\theta_i \right>\right)~dt \right) \, dy\notag \\
        &=\sum_{i=1}^k \int_{\mathbb{R}^d}  |f(y)| \cdot \alpha(h(y),\mathcal{L})_i \, dy\notag \\
        &= \int_{\mathbb{R}^d}  |f(y)| \cdot \left(\sum_{i=1}^k \alpha(h(y),\mathcal{L})_i \right) \, dy\notag \\
        &= \int_{\mathbb{R}^d}  |f(y)| \, dy \notag \\
        & = \|f \|_1.
    \end{align}
So $\mathcal{H}^\alpha_{\mathcal{L}} f \in L^1(\mathcal{L})$. It implies the operator $\mathcal{H}_{\mathcal{L}}^{\alpha} \colon L^1(\mathbb{R}^{d}) \to L^1(\mathcal{L})$ is well-defined, as well as $\mathcal{H}^{\alpha}$.
\end{proof}
\begin{remark}
    Note that, from the above proof, we see that if $f \in \mathcal{P}(\mathbb{R}^d)$, i.e. $f \in L^1(\mathbb{R}^d)$, $\|f \|_1 = 1$ and $f(y) \ge 0$ for all $ y \in \mathbb{R}^d$, we also have $\|\mathcal{H}_{\mathcal{L}}^{\alpha}f\|_{\mathcal{L}} =1$ and $\mathcal{H}^\alpha_{\mathcal{L}} f(x_i+t\cdot \theta_i) \ge 0$ for all $x_i+t\cdot \theta_i \in \mathcal{L}$. It implies that $\mathcal{H}_{\mathcal{L}}^{\alpha}f \in \mathcal{P}(\mathcal{L})$.
\end{remark}

\subsection{Proof for Theorem~\ref{mainbody:injectivity-of-Radon-Transform-circular}}

\label{appendix:Proof for Theorem {mainbody:injectivity-of-Radon-Transform-circular}}

Recall the original Circular Radon Transform $\mathcal{CR}$ \cite{kuchment-2006, kolouri2019generalized} as follows:
    \begin{align}
        \mathcal{CR} ~\colon ~~~~~~~~~  L^1(\mathbb{R}^d) ~~~~&\longrightarrow ~ L^1(\mathbb{R} \times \mathbb{S}^{d-1} \times \mathbb{R}_{\ge 0} )\notag \\ f ~~~~~~~~  &\longmapsto~~~~~~~~~~~~ \mathcal{CR}f, 
    \end{align}
    where:
    \begin{align}
    \mathcal{CR}f ~ \colon ~~~~~~~~ \mathbb{R} \times \mathbb{S}^{d-1} \times \mathbb{R}_{\ge 0}  ~~~~&\longrightarrow ~~~~~~~~~~~\mathbb{R} \\
         (t,\theta,r) ~~~~~~~~~~&\longmapsto ~~\int_{\mathbb{R}^d} f(y) \cdot \delta\left(t - \|y-r\theta\|_2\right)~dy.
    \end{align}
In \cite{kuchment-2006}, it is showed that the Circular Radon Transform $\mathcal{CR}$ is injective. We will leverage this result to prove the injectivity of the proposed Circular Radon Transform on Systems of Lines $\mathcal{CR}^\alpha$. 

First, for each $\theta \in \mathbb{S}^{d-1}$, consider the tree system $\mathcal{L}^{(i)}$ consists of $k$ identical lines $(0,\theta)$. Define the function $g$ as follows:
\begin{align}\label{appendix:eq-definition-g}
    g ~ \colon ~~~~~~~~ \mathbb{R} \times \mathbb{S}^{d-1} \times \mathbb{R}_{\ge 0}  ~~~~&\longrightarrow ~~~~~~~~~~~\mathbb{R} \\
         (t,\theta,r) ~~~~~~~~~~&\longmapsto ~~ \sum_{j=1}^k \mathcal{CR}^\alpha_{\mathcal{L}^{(j)},r}f(x_{\mathcal{L}^{(j)}:i} + t \cdot \theta_{\mathcal{L}^{(j)}:i} ) .
    \end{align}
Since 
\begin{align}
    \mathcal{CR}_{\mathcal{L},r }^{\alpha}f(x_i+t \cdot \theta_i) = \int_{\mathbb{R}^d} f(y) \cdot \alpha(y,\mathcal{L})_i \cdot \delta\left(t - \|y-x_i-r\theta_i\|_2\right)~dy.
\end{align}
We have:
\allowdisplaybreaks
\begin{align}
    g(t,\theta,r) &= \sum_{j=1}^k \mathcal{CR}^\alpha_{\mathcal{L}^{(j)},r}f(x_{\mathcal{L}^{(j)}:i} + t \cdot \theta_{\mathcal{L}^{(j)}:i} ) \\
    &=\sum_{i=1}^k \int_{\mathbb{R}^d} f(y) \cdot \alpha(y,\mathcal{L}^{(j)})_i \cdot \delta\left(t - \|y-x_{\mathcal{L}^{(j)}:i}-r\theta_{\mathcal{L}^{(j)}:i}\|_2\right)~dy\\
    &= \sum_{i=1}^k \int_{\mathbb{R}^d} f(y) \cdot \alpha(y,\mathcal{L}^{(j)})_i \cdot \delta\left(t - \|y-r\theta\|_2\right)~dy \\
    &= \sum_{i=1}^k \int_{\mathbb{R}^d} f(y) \cdot \delta\left(t - \|y-r\theta\|_2\right)\cdot \left(\sum_{i=1}^k\alpha(y,\mathcal{L}^{(j)})_i \right)~dy  \\
    &= \sum_{i=1}^k \int_{\mathbb{R}^d} f(y) \cdot \delta\left(t - \|y-r\theta\|_2\right)\cdot  ~dy  \\
    &= \mathcal{CR}f.
\end{align}
It is clear that $\mathcal{CR}^\alpha$ is a linear operator. To prove $\mathcal{CR}^\alpha$ injective, consider $f \in \operatorname{Ker}(\mathcal{CR}^\alpha)$. By the definition of $g$ in Eq.~\eqref{appendix:eq-definition-g}, we have $g$ is the function $0$. But $g$ is exactly is the Circular Radon Transform of $f$, and since the Circular Radon Transform is injective, we conclude that $f$ is the function $0$. In conclusion, $\mathcal{CR}^\alpha$ is injective.

\subsection{Proof for Theorem~\ref{mainbody:injectivity-of-Radon-Transform-spatial}}
\label{Proof for Theorem {mainbody:injectivity-of-Radon-Transform-spatial}}

We present the proof for Theorem~\ref{mainbody:injectivity-of-Radon-Transform-spatial}. 
\begin{proof}
    Recall the Radon Transform on Systems of Lines $\mathcal{R}^\alpha$ \cite{tran2025distancebased} as follows:
    \begin{align}
        \mathcal{R}^{\alpha}~ \colon ~ L^1(\mathbb{R}^d)~  &\longrightarrow ~ \prod_{\mathcal{L} \in \mathbb{L}^d_k} L^1(\mathcal{L}) \notag \\ f ~~~~ &\longmapsto ~ \left(\mathcal{R}_{\mathcal{L}}^{\alpha}f\right)_{\mathcal{L} \in \mathbb{L}^d_k},
    \end{align} 
where
    \begin{align}
\mathcal{R}_{\mathcal{L}}^{\alpha}f~ \colon ~~~~~~~~~~~~\mathcal{L}~~~~~~~ &\longrightarrow ~~~~~~~\mathbb{R}  \notag\\x_i+t \cdot \theta_i ~~&\longmapsto ~\int_{\mathbb{R}^d} f(y) \cdot \alpha(y,\mathcal{L})_i \cdot \delta\left(t - \left<y-x_i,\theta_i \right>\right)  ~ dy,
\end{align}
It is proved in \cite{tran2025distancebased} that $\mathcal{R}^\alpha$ is injective for $\operatorname{E}(d)$-invariant splitting map $\alpha$. We leverage this result to prove the Spatial Radon Transform on Systems of Lines is injective. First, by the injective continuous map $h ~ \colon ~ \mathbb{R}^d \rightarrow \mathbb{R}^{d_\theta}$, we show that the push-forward of $f \in \mathbb{R}^d$ via $h$, defined as:
\begin{align} \label{appendix:pushforward-in-spatial}
    h_\sharp f(y) =  \begin{cases}
        f(h^{-1}(y)) &, ~ \text{ for all } y \in \mathbb{R}^{d_\theta}~ \text{ such that }  y \in h(\mathbb{R}^d), \\
        0 &, ~ \text{ for all } y \in \mathbb{R}^{d_\theta}~ \text{ such that }  y \notin h(\mathbb{R}^d).
    \end{cases}
\end{align}
has its Radon Transform on Systems of Lines, i.e. $\{\mathcal{R}^\alpha_\mathcal{L}(h_\sharp f)\}_{\mathcal{L} \in \mathbb{L}^{d_\theta}_k}$, equal to the Spatial Radon Transform on Systems of Lines of $f$, i.e. $\{\mathcal{H}^\alpha_\mathcal{L}f\}_{\mathcal{L} \in \mathbb{L}^{d_\theta}_k}$. In other words, for all $\mathcal{L} \in \mathbb{L}^{d_\theta}_k$, we have:
\begin{align}
    \mathcal{H}^\alpha_\mathcal{L}f = \mathcal{R}^\alpha_\mathcal{L}(h_\sharp f).
\end{align}
Indeed, we have:
\allowdisplaybreaks
\begin{align}
    \mathcal{R}^\alpha_\mathcal{L}(h_\sharp f)(x_i+t\cdot \theta_i) &= \int_{\mathbb{R}^{d_{\theta}}} h_\sharp f(y) \cdot \alpha(y,\mathcal{L})_l \cdot \delta\left(t - \left<y-x_i,\theta_i \right>\right)  ~ dy  \notag\\
    &= \int_{h(\mathbb{R}^{d})} h_\sharp f(y) \cdot \alpha(y,\mathcal{L})_l \cdot \delta\left(t - \left<y-x_i,\theta_i \right>\right)  ~ dy  \notag\\
    &= \int_{\mathbb{R}^{d}} f(h^{-1}(h(y)) \cdot \alpha(h(y),\mathcal{L})_l \cdot \delta\left(t - \left<h(y)-x_i,\theta_i \right>\right)  ~ dy  \notag\\
    &= \int_{\mathbb{R}^d} f(y) \cdot \alpha(h(y),\mathcal{L})_l \cdot \delta\left(t - \left<h(y)-x_i,\theta_i \right>\right)  ~ dy  \notag\\
    &= \mathcal{H}^\alpha_\mathcal{L}f(x_i+t\cdot \theta_i).
\end{align}
It is clear that $\mathcal{H}^\alpha$ is a linear operator. To prove $\mathcal{H}^\alpha$ is injective, consider $f \in \operatorname{Ker}(\mathcal{H}^\alpha)$. Since $\mathcal{H}^\alpha_\mathcal{L}f = \mathcal{R}^\alpha_\mathcal{L}(h_\sharp f)$, it implies that $h_\sharp f \in \operatorname{Ker}(\mathcal{R}^\alpha)$. Since $\mathcal{R}^\alpha$ is injective, it implies that $h_\sharp f$ is the function $0$. By the definition of the push-forward $h_\sharp f$ as in Eq.~\eqref{appendix:pushforward-in-spatial}, we conclude that  $f$ is the function $0$. In conclusion, $\mathcal{H}^\alpha$ is injective.
\end{proof}

\subsection{Proof of Theorem~\ref{thm:CircularTSW and SpatialTSW are metrics}}
\label{appendix:thm:CircularTSW and SpatialTSW are metrics}

We show that CircularTSW is a metric on $\mathcal{P}(\mathbb{R}^d)$. The proof for SpatialTSW is similar.
\begin{proof}
    We will show that:
    \begin{equation}
    \text{CircularTSW} (\mu,\nu) = 
  \int_{\mathbb{T}^d_k} \text{W}(\mathcal{CR}^\alpha_{\mathcal{L},r} f_\mu, \mathcal{CR}^\alpha_{\mathcal{L},r} f_\nu) ~d\sigma(\mathcal{L}),
\end{equation}
is a metric on $\mathcal{P}(\mathbb{R}^d)$, by verifying its positive definiteness, symmetry and triangle inequality.
\paragraph{Positive definiteness.} For $\mu,\nu \in \mathcal{P}(\mathbb{R}^d)$, it is clear that 
\begin{align}
    \text{CircularTSW}(\mu,\mu) = 0,
\end{align} 
and
\begin{align}
\text{CircularTSW}(\mu,\nu) \ge 0. 
\end{align}
If $\text{CircularTSW}(\mu,\nu) = 0$, then $\text{W}(\mathcal{CR}^\alpha_{\mathcal{L},r} f_\mu, \mathcal{CR}^\alpha_{\mathcal{L},r} f_\nu) = 0$ for almost every $\mathcal{L} \in \mathbb{T}^d_k$. Since $\text{W}$ is a metric on $\mathcal{P}(\mathcal{L})$, we have $\mathcal{CR}^\alpha_{\mathcal{L},r} f_\mu =  \mathcal{CR}^\alpha_{\mathcal{L},r} f_\nu$ for almost every $\mathcal{L} \in \mathbb{T}$. By Theorem~\ref{mainbody:injectivity-of-Radon-Transform-circular}, it implies that $\mu =\nu$.

\paragraph{Symmetry.} For $\mu,\nu \in \mathcal{P}(\mathbb{R}^d)$, we have:
\begin{align}
    \text{CircularTSW} (\mu,\nu) &= 
  \int_{\mathbb{T}^d_k} \text{W}(\mathcal{CR}^\alpha_{\mathcal{L},r} f_\mu, \mathcal{CR}^\alpha_{\mathcal{L},r} f_\nu) ~d\sigma(\mathcal{L}) \notag \\
  &= 
  \int_{\mathbb{T}^d_k} \text{W}(\mathcal{CR}^\alpha_{\mathcal{L},r} f_\nu, \mathcal{CR}^\alpha_{\mathcal{L},r} f_\mu) ~d\sigma(\mathcal{L}) \notag \\
  &= \text{CircularTSW} (\nu,\mu)
\end{align}

So  $\text{CircularTSW}(\mu,\nu)  =\text{CircularTSW} (\nu,\mu)$. 

\paragraph{Triangle inequality.} For $\mu_1, \mu_2, \mu_3 \in \mathcal{P}(\mathbb{R}^D)$, we have:
\begin{align}
    &\text{CircularTSW} (\mu_1,\mu_2) + \text{CircularTSW} (\mu_2,\mu_3) \notag\\
    &\hspace{30pt}=\int_{\mathbb{T}^d_k} \text{W}(\mathcal{CR}^\alpha_{\mathcal{L},r} f_{\mu_1}, \mathcal{CR}^\alpha_{\mathcal{L},r} f_{\mu_2})  ~d\sigma(\mathcal{L})
 +\int_{\mathbb{T}^d_k} \text{W}(\mathcal{CR}^\alpha_{\mathcal{L},r} f_{\mu_2}, \mathcal{CR}^\alpha_{\mathcal{L},r} f_{\mu_3})~d\sigma(\mathcal{L}) \notag\\
 & \hspace{30pt}= \int_{\mathbb{T}^d_k} \left(\text{W}(\mathcal{CR}^\alpha_{\mathcal{L},r} f_{\mu_1}, \mathcal{CR}^\alpha_{\mathcal{L},r} f_{\mu_2}) +\text{W}(\mathcal{CR}^\alpha_{\mathcal{L},r} f_{\mu_1}, \mathcal{CR}^\alpha_{\mathcal{L},r} f_{\mu_2}) \right)  ~d\sigma(\mathcal{L}) \notag\\
 & \hspace{30pt}\ge \int_{\mathbb{T}^d_k} \text{W}(\mathcal{CR}^\alpha_{\mathcal{L},r} f_{\mu_1}, \mathcal{CR}^\alpha_{\mathcal{T},r} f_{\mu_3})~d\sigma(\mathcal{L})\notag \\
 &\hspace{30pt} = \text{CircularTSW} (\mu_1,\mu_3).
\end{align}
The triangle inequality holds for $\text{CircularTSW}$.

In conclusion, CircularTSW is a metric on the space $\mathcal{P}(\mathbb{R}^d)$.
\end{proof}

\section{Radon Transform on Systems of Lines for Spherical Functions}\label{app:sec:radon_sphere}

In this section, we review \cite{tran2025spherical} which proposes Spherical Radon Transform on Spherical Trees and Spherical Tree-Sliced Wasserstein distance, which are analogs to Radon Transform on Systems of Lines \cite{tran2024tree,tran2025distancebased} and corresponding metric, applied for spherical functions. Then, we explain how to apply Generalized-like framework in this paper for spherical settings
\subsection{Background for Spherical Radon Transform on Spherical Trees}
\label{appendix:background-stsw}
To make this easy to follow, we will follow the construction of Appendix~\ref{appendix:sec:background-tsw-sl}.
\textbf{Building blocks of Spherical Tree-Sliced Wasserstein distance.}
\begin{enumerate}
    \item Given a positive number $d$ presenting the dimension. We will work with functions on the d-dimensional hypersphere $\mathbb{S}^d \subset \mathbb{R}^{d+1}$, where:
\begin{align*}
    \mathbb{S}^{d} \coloneqq \left\{x = (x_0, x_1,\ldots, x_d) \in \mathbb{R}^{d+1} ~  : ~ \|x\|_2 = 1 \right\} \subset \mathbb{R}^{d+1}.
\end{align*}
Note that $\mathbb{S}^d$ is a metric space with the metric  $d_{\mathbb{S}^d}$ defined as  $d_{\mathbb{S}^d}(a,b) = \operatorname{arccos}\left<a,b\right>_{\mathbb{R}^{d+1}}$.
    \item The stereographic projection corresponding to $x \in \mathbb{S}^{d}$ is defined by:
\begin{align}
    \varphi_x ~ \colon ~~~~ \mathbb{S}^{d} \setminus \{x\} ~&\longrightarrow ~ H_x \nonumber \\
    y ~~~ &\longmapsto ~ \dfrac{-\left<x,y \right>}{1-\left<x,y \right>} \cdot x + \dfrac{1}{1-\left<x,y \right>} \cdot y.
\end{align}
By convention, let $\varphi_x(x) = \infty$, then $ \varphi_x \colon\mathbb{S}^{d} \rightarrow H_x \cup \{\infty\}$.
\item The \textit{spherical ray} with root $x$ and direction $y$, denoted by $r^x_y$, is defined as 
\begin{align}
    r_y^x = \varphi_x^{-1}\bigl(\{ t\cdot y ~ \colon ~ t > 0\} \cup \{\infty\} \bigr).
\end{align}
Each ray $r_y^x$ is isomorphic to $[0,\pi]$ via $d_{\mathbb{S}^d}(x,\cdot)$, so it is parameterized as $(t,r^x_y)$.

\item Spherical trees $\mathcal{T}^x_{y_1,\ldots,y_k}$ in $\mathbb{S}^d$ is the gluing space of $k$ spherical rays  $r^x_{y_i}$ at the root $x$. $x$ is the root and $y_1,\ldots,y_k$ are the edges of $\mathcal{T}^x_{y_1,\ldots,y_k}$. It is a measure metric space, endowed with tree metric.

\item The space of spherical trees with $k$ edges in $\mathbb{S}^d$ is denoted by $\mathbb{T}^d_k$, with a probability
distribution $\sigma$ on $\mathbb{T}^d_k$, which comes from the tree sampling process.

\item For $\mathcal{T} \in \mathbb{T}^d_k$, \textit{the space of integrable functions on $\mathcal{T}$} is:
    \begin{align}
    L^1(\mathcal{T}) = \left \{ f \colon \mathcal{T} \rightarrow \mathbb{R} ~ \colon ~ \|f\|_{\mathcal{L}} = \sum_{i=1}^k \int_{0}^\pi |f(t,r^x_y)|  \, dt < \infty \right \}.
    \end{align}

\item A splitting map $\alpha$ is a continuous map from $\mathbb{S}^d \times \mathbb{T}^d_k$ to the $(k-1)$-dimensional standard simplex $\Delta_{k-1}$, i.e. $\alpha \in \mathcal{C}\left(\mathbb{S}^d \times \mathbb{T}^d_k, \Delta_{k-1} \right)$. For $f \in L^1(\mathbb{S}^d)$, we define:
\begin{align}
    \mathcal{R}^\alpha_{\mathcal{T}}f ~~~~~ \colon~~~~~~ \mathcal{T} ~~~~~~~~&\longrightarrow ~~~~~~~~\mathbb{R} \\
    (t,r^x_{y_i})  ~~&\longmapsto ~~~ \int_{\mathbb{S}^d} f(y) \cdot \alpha(y, \mathcal{T})_i \cdot \delta(t  - \operatorname{arccos}\left<x,y \right>) ~ dy.
\end{align}
The function $\mathcal{R}^\alpha_{\mathcal{T}}f$ is in $L^1(\mathcal{T})$.

\item The operator:
\begin{align*}
        \mathcal{R}^{\alpha}~ \colon ~ L^1(\mathbb{S}^d)~  &\longrightarrow ~ \prod_{\mathcal{T} \in \mathbb{T}^d_k} L^1(\mathcal{T}) \\ f ~~~~ &\longmapsto ~ \left(\mathcal{R}_{\mathcal{T}}^{\alpha}f\right)_{\mathcal{T} \in \mathbb{T}^d_k}.
    \end{align*}
    is called the \textit{Spherical Radon Transform on Spherical Trees}.

\item When the splitting map $\alpha$ is $\operatorname{O}(d+1)$-invariant (this $\operatorname{O}(d+1)$-invariance will be described in the next part), the Spherical Radon Transform on Spherical Trees is injective (see \cite{tran2025spherical}).

\item  The \textit{Spherical Tree-Sliced Wasserstein Distance} \cite{tran2025spherical} between $\mu,\nu$ in $\mathcal{P}(\mathbb{S}^d)$ is defined by:
\begin{equation} \label{appendix:eq-STSW}
    \text{STSW} (\mu,\nu) = 
 \int_{\mathbb{T}^d_k} \text{W}_{d_\mathcal{T},1}(\mathcal{R}^\alpha_\mathcal{T} f_\mu, \mathcal{R}^\alpha_\mathcal{T} f_\nu) ~d\sigma(\mathcal{T}).
\end{equation}

\item The $\text{STSW}$ distance is a metric on $\mathcal{P}(\mathbb{S}^d)$.

\item It is worth noting that, on tree systems, optimal transport problems admits closed-form expression, since it is a metric space with tree metric (see \cite{le2019tree}). Leveraging this closed-form expression and the Monte Carlo method, the
distance in Eq.~\eqref{appendix:eq-STSW} can be efficiently approximated by a closed-form expression.
\end{enumerate}

\textbf{The group $\operatorname{O}(d+1)$ and its actions.} The orthogonal group $\operatorname{O}(d+1)$ is the group of linear transformations of $\mathbb{R}^{d+1}$ that preserves the Euclidean norm $\| \cdot \|_2$. The group $\operatorname{O}(d+1)$ acts on $\mathbb{S}^d$ naturally as follows: For $x \in \mathbb{S}^d$ and $g=Q \in \operatorname{O}(d+1)$, we have:
\begin{align}
    (g,x) \longmapsto gx = Q \cdot x.
\end{align}
It naturally induces a group action on the set of all spherical lines in $\mathbb{S}^d$, as well as spherical trees.
The tree structure of a spherical tree is preserved under the action of $\operatorname{O}(d+1)$ (see  \cite{tran2024tree, tran2025spherical}). In other words, if $\mathcal{T} \in \mathbb{T}$ is a spherical tree, then $g\mathcal{T}$ is also a spherical tree. 

\begin{definition}
    A splitting map $\alpha$ in $\mathcal{C}(\mathbb{S}^d \times \mathbb{T}^d_k, \Delta_{k-1})$ is said to be $\operatorname{O}(d+1)$-invariant, if we have
    \begin{align}
        \alpha(gy,g\mathcal{T}) = \alpha(y,\mathcal{T})
    \end{align}
    for all $(y, \mathcal{T}) \in \mathbb{S}^d \times \mathbb{T}^d_k$ and $g \in \operatorname{O}(d+1)$.
\end{definition}

A candidate for $\operatorname{O}(d+1)$-invariant splitting maps is presented as follows: Consider the map $\beta \colon \mathbb{S}^d \times \mathbb{T}^d_k \rightarrow \mathbb{R}^k$:
\begin{align}
    \beta(y,\mathcal{T}^x_{y_1,\ldots,y_k})_i = \begin{cases}
        0, & \text{if $y=x$ or $y=-x$,} \\
        \operatorname{arccos}\left(\dfrac{\left<y,y_i\right>}{\sqrt{1-\left<x,y\right>^2}}\right) \cdot \sqrt{1-\left<x,y\right>^2},  & \text{if $y \neq \pm x$.}
    \end{cases} 
\end{align}
The map $\beta$ is continuous and $\operatorname{O}(d+1)$-invariant. Take $\alpha \colon \mathbb{S}^d \times \mathbb{T}^d_k \rightarrow \Delta_{k-1}$ to be:
\begin{align}
\label{eq:construction-of-alpha}
    \alpha(y,\mathcal{T}) = \operatorname{softmax}\Bigl (\{ \beta(y,\mathcal{T})_i\}_{i=1,\ldots,k} \Bigr ).
\end{align}

\subsection{Spatial Spherical Radon Transform on Spherical Trees}
\label{appendix:def-Spatial Spherical Radon Transform on Spherical Trees}
Consider a positive integer $d_\theta$, and an injective continuous map $h \colon \mathbb{S}^d \rightarrow \mathbb{S}^{d_\theta}$, and a splitting map $\alpha \in \mathcal{C}(\mathbb{R}^{d_\theta} \times \mathbb{L}^{d_\theta}_k, \Delta_{k-1})$ defining the splitting mechanism. Let $\mathcal{T}$ be a spherical tree of $k$ edges in $\mathbb{T}^{d_\theta}_k$. For a function $f \in L^1(\mathbb{S}^d)$, define the function $\mathcal{H}_{\mathcal{T}}^{\alpha}f \in L^1(\mathcal{T})$ as follows:
\begin{align}
\mathcal{R}^\alpha_{\mathcal{T}}f ~~~~ \colon ~~~~~ \mathcal{T}~~~~~~~~~ &\longrightarrow ~~~~~~~ \mathbb{R} \\
(t,r^x_{y_i}) ~~~~~~ &\longmapsto ~~ \int_{\mathbb{S}^d} f(y) \cdot \alpha(h(y), \mathcal{T})_i \cdot \delta(t  - \operatorname{arccos}\left<x,h(y) \right>) ~ dy,
\end{align}
The \textit{Spatial Spherical Radon Transform on Spherical Trees} is defined as the operator:
\begin{align}
        \mathcal{H}^{\alpha} ~~ \colon~~~~~~~~~ L^1(\mathbb{R}^d) ~~~~~~ &\longrightarrow ~~~~~~\prod_{\mathcal{T} \in \mathbb{T}^{d_\theta}_k} L^1(\mathcal{T}) \notag \\f  ~~~~~~~~~~~~&\longmapsto ~~~~~~ \left(\mathcal{H}_{\mathcal{T}}^{\alpha}f\right)_{\mathcal{T} \in \mathbb{T}^{d_\theta}_k}.
\end{align}

\subsection{Proof for Theorem~\ref{mainbody:injectivity-of-generalized-spherical-radon-transform-tree}}
\label{appendix:Proof for Theorem{mainbody:injectivity-of-generalized-spherical-radon-transform-tree}}
We present the proof for Theorem~\ref{mainbody:injectivity-of-generalized-spherical-radon-transform-tree} about the injectivity of the Spatial Spherical Radon Transform on Spherical Trees.
\begin{proof}
    Recall the Radon Transform on Spherical Tres $\mathcal{R}^\alpha$ \cite{tran2025spherical} as follows:
    \begin{align}
        \mathcal{R}^{\alpha}~ \colon ~ L^1(\mathbb{S}^d)~  &\longrightarrow ~ \prod_{\mathcal{T} \in \mathbb{T}^d_k} L^1(\mathcal{T}) \notag \\ f ~~~~ &\longmapsto ~ \left(\mathcal{R}_{\mathcal{T}}^{\alpha}f\right)_{\mathcal{T} \in \mathbb{T}^d_k},
    \end{align} 
where
    \begin{align}
\mathcal{R}_{\mathcal{T}}^{\alpha}f~ \colon ~~~~~~~~~~~~\mathcal{T}~~~~~~~ &\longrightarrow ~~~~~~~\mathbb{R}  \notag\\(t,r^x_{y_i}) ~~&\longmapsto ~\int_{\mathbb{S}^d} f(y) \cdot \alpha(y,\mathcal{L})_1 \cdot \delta\left(t - \arccos\left<x,y\right>\right)  ~ dy,
\end{align}
It is proved in \cite{tran2025spherical} that $\mathcal{R}^\alpha$ is injective for $\operatorname{O}(d+1)$-invariant splitting map $\alpha$. We use this result to prove the Spatial Radon Transform on Spherical Trees is injective. First, by the injective continuous map $h ~ \colon ~ \mathbb{S}^d \rightarrow \mathbb{S}^{d_\theta}$, we show that the push-forward of $f \in \mathbb{R}^d$ via $h$, defined as:
\begin{align} \label{appendix:pushforward-in-spatial-spherical}
    h_\sharp f(y) =  \begin{cases}
        f(h^{-1}(y)) &, ~ \text{ for all } y \in \mathbb{S}^{d_\theta}~ \text{ such that }  y \in h(\mathbb{S}^d), \\
        0 &, ~ \text{ for all } y \in \mathbb{S}^{d_\theta}~ \text{ such that }  y \notin h(\mathbb{S}^d).
    \end{cases}
\end{align}
has its Spherical Radon Transform on Spherical Trees, i.e. $\{\mathcal{R}^\alpha_\mathcal{T}(h_\sharp f)\}_{\mathcal{T} \in \mathbb{T}^{d_\theta}_k}$, equal to the Spatial Radon Transform on Spherical Trees of $f$, i.e. $\{\mathcal{H}^\alpha_\mathcal{T}f\}_{\mathcal{T} \in \mathbb{T}^{d_\theta}_k}$. In other words, for all $\mathcal{T} \in \mathbb{T}^{d_\theta}_k$, we have:
\begin{align}
    \mathcal{H}^\alpha_\mathcal{T}f = \mathcal{R}^\alpha_\mathcal{T}(h_\sharp f).
\end{align}
Indeed, we have:
\allowdisplaybreaks
\begin{align}
    \mathcal{R}^\alpha_\mathcal{T}(h_\sharp f)(t,r^{x}_{y_i}) &= \int_{\mathbb{S}^{d_{\theta}}} h_\sharp f(y) \cdot \alpha(y,\mathcal{L})_l \cdot \delta\left(t - \arccos\left<x,y \right>\right)  ~ dy \notag \\
    &= \int_{h(\mathbb{S}^{d})} h_\sharp f(y) \cdot \alpha(y,\mathcal{L})_l \cdot \delta\left(t - \arccos\left<x,y \right>\right)  ~ dy \notag \\
    &= \int_{\mathbb{S}^{d}} f(h^{-1}(h(y)) \cdot \alpha(h(y),\mathcal{L})_l \cdot \delta\left(t - \arccos\left<x,h(y) \right>\right)  ~ dy \notag\\
    &= \int_{\mathbb{S}^d} f(y) \cdot \alpha(h(y),\mathcal{L})_l \cdot \delta\left(t - \arccos\left<x,h(y) \right>\right)  ~ dy \notag \\
    &= \mathcal{H}^\alpha_\mathcal{T}f(t,r^{x}_{y_i}).
\end{align}
It is clear that $\mathcal{H}^\alpha$ is a linear operator. To prove $\mathcal{H}^\alpha$ is injective, consider $f \in \operatorname{Ker}(\mathcal{H}^\alpha)$. Since $\mathcal{H}^\alpha_\mathcal{T}f = \mathcal{R}^\alpha_\mathcal{T}(h_\sharp f)$, it implies that $h_\sharp f \in \operatorname{Ker}(\mathcal{R}^\alpha)$. Since $\mathcal{R}^\alpha$ is injective, it implies that $h_\sharp f$ is the function $0$. By the definition of the push-forward $h_\sharp f$ as in Eq.~\eqref{appendix:pushforward-in-spatial-spherical}, we conclude that  $f$ is the function $0$. In conclusion, $\mathcal{H}^\alpha$ is injective.
\end{proof}

\subsection{Spatial Spherical Tree-Sliced Wasserstein Distance}
\label{appendix:dis-Spatial Spherical Tree-Sliced Wasserstein distance}
For two probability measures $\mu$ and $\nu$ with density function $f_\mu$ and $f_\nu$. Given a positive integer $d_\theta$ and a choice of the continuous injective map $h \colon \mathbb{S}^d \rightarrow \mathbb{S}^{d_\theta}$, the \textit{Spatial Spherical Tree-Sliced Wasserstein Distance} between $\mu$ and $\nu$ is defined as the average Wasserstein distance on the tree-metric space $\mathcal{L}$ between the Spatial Spherical Radon Transform on Spherical Trees of $f_\mu$ and $f_\nu$. Following \cite{tran2025spherical}, this averaging is taken over the space of trees $\mathbb{T}^{d_\theta}_k$, according to a distribution $\sigma$ on $\mathbb{T}^{d_\theta}_k$ which arises from the tree sampling process.
\begin{definition}
The \textit{Spatial Spherical Tree-Sliced Wasserstein Distance} between $\mu$ and $\nu$ in $\mathcal{P}(\mathbb{S}^d)$ is defined by:
\begin{align} 
\label{eq:SpatialSTSW-formula} 
    \text{SpatialSTSW} (\mu,\nu) \coloneqq \int_{\mathbb{T}^{d_\theta}_k} \text{W}(\mathcal{H}^\alpha_\mathcal{T} f_\mu, \mathcal{H}^\alpha_\mathcal{T} f_\nu) ~d\sigma(\mathcal{T}).
\end{align}
\end{definition}
SpatialSTSW is a metric on the space $\mathcal{P}(\mathbb{S}^d)$ of measures on $\mathbb{S}^d$.
\begin{theorem}
    SpatialSTSW is a metric on the space $\mathcal{P}(\mathbb{S}^d)$.
\end{theorem}
\begin{proof}
    We will show that:
    \begin{equation}
    \text{SpatialSTSW} (\mu,\nu) = 
  \int_{\mathbb{T}^{d_{\theta}}_k} \text{W}(\mathcal{H}^\alpha_\mathcal{T} f_\mu, \mathcal{H}^\alpha_\mathcal{T} f_\nu) ~d\sigma(\mathcal{T}),
\end{equation}
is a metric on $\mathcal{P}(\mathbb{S}^d)$, by verifying its positive definiteness, symmetry and triangle inequality.
\paragraph{Positive definiteness.} For $\mu,\nu \in \mathcal{P}(\mathbb{S}^d)$, it is clear that 
\begin{align}
    \text{SpatialSTSW}(\mu,\mu) = 0,
\end{align} 
and
\begin{align}
\text{SpatialSTSW}(\mu,\nu) \ge 0. 
\end{align}
If $\text{SpatialSTSW}(\mu,\nu) = 0$, then $\text{W}(\mathcal{H}^\alpha_{\mathcal{T}} f_\mu, \mathcal{H}^\alpha_{\mathcal{T}} f_\nu) = 0$ for almost every $\mathcal{T} \in \mathbb{T}^d_k$. Since $\text{W}$ is a metric on $\mathcal{P}(\mathcal{T})$, we have $\mathcal{H}^\alpha_{\mathcal{T}} f_\mu =  \mathcal{H}^\alpha_{\mathcal{T}} f_\nu$ for almost every $\mathcal{L} \in \mathbb{T}$. By the injectivity of the Spatial Spherical Radon Transform on Spherical Trees, it implies that $\mu =\nu$.

\paragraph{Symmetry.} For $\mu,\nu \in \mathcal{P}(\mathbb{S}^d)$, we have:
\begin{align}
    \text{SpatialSTSW} (\mu,\nu) &= 
  \int_{\mathbb{T}^{d_\theta}_k} \text{W}(\mathcal{H}^\alpha_{\mathcal{L}} f_\mu, \mathcal{H}^\alpha_{\mathcal{L}} f_\nu) ~d\sigma(\mathcal{T}) \notag \\
  &= 
  \int_{\mathbb{T}^{d_\theta}_k} \text{W}(\mathcal{H}^\alpha_{\mathcal{L}} f_\nu, \mathcal{H}^\alpha_{\mathcal{L}} f_\mu) ~d\sigma(\mathcal{T}) \notag \\
  &= \text{SpatialSTSW} (\nu,\mu)
\end{align}

So  $\text{SpatialSTSW}(\mu,\nu)  =\text{SpatialSTSW} (\nu,\mu)$. 

\paragraph{Triangle inequality.} For $\mu_1, \mu_2, \mu_3 \in \mathcal{P}(\mathbb{S}^d)$, we have:
\begin{align}
    &\text{SpatialSTSW} (\mu_1,\mu_2) + \text{SpatialSTSW} (\mu_2,\mu_3) \notag\\
    &\hspace{30pt}=\int_{\mathbb{T}^{d_\theta}_k} \text{W}(\mathcal{H}^\alpha_{\mathcal{T}} f_{\mu_1}, \mathcal{H}^\alpha_{\mathcal{T}} f_{\mu_2})  ~d\sigma(\mathcal{T})
 +\int_{\mathbb{T}^{d_\theta}_k} \text{W}(\mathcal{H}^\alpha_{\mathcal{T}} f_{\mu_2}, \mathcal{H}^\alpha_{\mathcal{T}} f_{\mu_3})~d\sigma(\mathcal{T}) \notag\\
 & \hspace{30pt}= \int_{\mathbb{T}^{d_\theta}_k} \left(\text{W}(\mathcal{H}^\alpha_{\mathcal{T}} f_{\mu_1}, \mathcal{H}^\alpha_{\mathcal{T}} f_{\mu_2}) +\text{W}(\mathcal{H}^\alpha_{\mathcal{T}} f_{\mu_1}, \mathcal{H}^\alpha_{\mathcal{T}} f_{\mu_2}) \right)  ~d\sigma(\mathcal{T}) \notag\\
 & \hspace{30pt}\ge \int_{\mathbb{T}^{d_\theta}_k} \text{W}(\mathcal{H}^\alpha_{\mathcal{T}} f_{\mu_1}, \mathcal{H}^\alpha_{\mathcal{T}} f_{\mu_3})~d\sigma(\mathcal{T})\notag \\
 &\hspace{30pt} = \text{CircularTSW} (\mu_1,\mu_3).
\end{align}
The triangle inequality holds for $\text{SpatialSTSW}$.

In conclusion, SpatialSTSW is a metric on the space $\mathcal{P}(\mathbb{S}^d)$.
\end{proof}

\textbf{The choice of the injective map $h$.} Note that, the map $h ~ \colon ~ \mathbb{S}^d \rightarrow \mathbb{S}^{d_\theta}$ has to satisfy the injective condition. We construct $h$ as follows. We construct $h$ as follows. First, consider $d_\theta = d + 1$. We define a continuous function:  
$k(y) = \frac{\pi}{2(1+\epsilon)} \left( \frac{1}{d+1} \sum_{i=0}^{d} y_i + 1 + \epsilon \right),$
which maps $y \in \mathbb{S}^d$ to the range $(0, \pi)$. We set $\epsilon = 10^{-6}$.  Using this, we define the mapping  
$h(y) = \left( \cos(k(y)), \sin(k(y)) \cdot y \right),$
which is injective.

\section{Experimental Details}\label{app:sec:experiment_details}

\subsection{Algorithm of proposed Tree-Sliced Distances}\label{app:subsec:alg_TSW}

We describe the pseudo-codes for $\widehat{\text{CircularTSW}}, \widehat{\text{SpatialTSW}}, \widehat{\text{SpatialSTSW}}$ in Algorithms~\ref{alg:compute-tsw-sl-circular},~\ref{alg:compute-tsw-sl-spatial},~\ref{alg:compute-tsw-sl-spatial-spherical} respectively.

\begin{algorithm}[H]
\caption{Circular Tree-Sliced Wasserstein distance.}
\begin{algorithmic}
\label{alg:compute-tsw-sl-circular}
    \STATE \textbf{Input:} Probability measures $\mu$ and $\nu$ in $\mathcal{P}(\mathbb{R}^d)$, number of tree systems $L$, number of lines in tree system $k$, space of tree systems $\mathbb{T}$, splitting maps $\alpha$, and parameter $r \in \mathbb{R}_{\ge0}$.
    \FOR{$i=1$ to $L$}
    \STATE Sampling $x \in \mathbb{R}^d$ and $\theta_1, \ldots, \theta_k \overset{i.i.d}{\sim} \mathcal{U}(\mathbb{S}^{d-1})$.
    \STATE Contruct tree system $\mathcal{L}_i = \{(x,\theta_1), \ldots, (x,\theta_k)\}$.
    \STATE Projecting $\mu$ and $\nu$ onto $\mathcal{L}_i$ to get $\mathcal{CR}^{\alpha}_{\mathcal{L}_i, r} \mu$ and $\mathcal{CR}^{\alpha}_{\mathcal{L}_i, r} \nu$.
  
    \STATE Compute $\widehat{\text{CircularTSW}} (\mu,\nu) =  (1/L) 
    \cdot \text{W}(\mathcal{CR}^\alpha_{\mathcal{L}_i, r} \mu,\mathcal{CR}^\alpha_{\mathcal{L}_i, r} \nu)$.
  \ENDFOR
 \STATE \textbf{Return:} $\widehat{\text{CircularTSW}} (\mu,\nu)$.
\end{algorithmic}
\end{algorithm}

\begin{algorithm}[H]
\caption{Spatial Tree-Sliced Wasserstein distance.}
\begin{algorithmic}
\label{alg:compute-tsw-sl-spatial}
    \STATE \textbf{Input:} Probability measures $\mu$ and $\nu$ in $\mathcal{P}(\mathbb{R}^{d_\theta})$, number of tree systems $L$, number of lines in tree system $k$, space of tree systems $\mathbb{T}$, splitting maps $\alpha$, and injective continuous map $h ~ \colon ~ \mathbb{R}^d \rightarrow \mathbb{R}^{d_\theta}$.
    \FOR{$i=1$ to $L$}
    \STATE Sampling $x \in \mathbb{R}^d$ and $\theta_1, \ldots, \theta_k \overset{i.i.d}{\sim} \mathcal{U}(\mathbb{S}^{{d_\theta}-1})$.

    \STATE Contruct tree system $\mathcal{L}_i = \{(x,\theta_1), \ldots, (x,\theta_k)\}$.
    \STATE Projecting $\mu$ and $\nu$ onto $\mathcal{T}_i$ to get $\mathcal{H}^{\alpha}_{\mathcal{L}_i} \mu$ and $\mathcal{H}^{\alpha}_{\mathcal{L}_i} \nu$.
  
    \STATE Compute $\widehat{\text{SpatialTSW}} (\mu,\nu) =  (1/L) 
    \cdot \text{W}(\mathcal{R}^\alpha_{\mathcal{L}_i} \mu,\mathcal{R}^\alpha_{\mathcal{L}_i} \nu)$.
  \ENDFOR
 \STATE \textbf{Return:} $\widehat{\text{SpatialTSW}} (\mu,\nu)$.
\end{algorithmic}
\end{algorithm}

\begin{algorithm}[H]
\caption{Spatial Spherical Tree-Sliced Wasserstein distance.}
\begin{algorithmic}
\label{alg:compute-tsw-sl-spatial-spherical}
    \STATE \textbf{Input:} Probability measures $\mu$ and $\nu$ in $\mathcal{P}(\mathbb{S}^d)$, number of tree systems $L$, number of lines in tree system $k$, space of tree systems $\mathbb{T}$, splitting maps $\alpha$, and injective continuous map $h ~ \colon ~ \mathbb{S}^d \rightarrow \mathbb{S}^{d_\theta}$.
    \FOR{$i=1$ to $L$}
    \STATE Sampling $x \in \mathbb{S}^{d_\theta}$ and $y_1, \ldots, y_k \overset{i.i.d}{\sim} \mathcal{U}(\mathbb{S}^{{d_\theta}-1})$.
    \STATE Contruct tree system $\mathcal{T}_i = \{(x,y_1), \ldots, (x,y_k)\}$.
    \STATE Projecting $\mu$ and $\nu$ onto $\mathcal{L}_i$ to get $\mathcal{H}^{\alpha}_{\mathcal{T}_i} \mu$ and $\mathcal{H}^{\alpha}_{\mathcal{T}_i} \nu$.
  
    \STATE Compute $\widehat{\text{SpatialSTSW}} (\mu,\nu) =  (1/L) 
    \cdot \text{W}(\mathcal{H}^\alpha_{\mathcal{T}_i} \mu,\mathcal{H}^\alpha_{\mathcal{T}_i} \nu)$.
  \ENDFOR
 \STATE \textbf{Return:} $\widehat{\text{SpatialSTSW}} (\mu,\nu)$.
\end{algorithmic}
\end{algorithm}

\subsection{Computational and Memory Complexity}
\label{appendix:complexity}

\begin{table}[t]
\caption{Computational and Memory Complexity Analysis of proposed Tree-Sliced distances}
\label{tab:complexity_analysis}
\vskip 0.15in
\begin{center}
\begin{small}
\begin{adjustbox}{width=0.95\textwidth}
\begin{tabular}{p{3cm}p{3cm}llll}
\toprule
\textbf{Distance} & \textbf{Operation} & \textbf{Description} & \textbf{Computation} & \textbf{Memory} \\ 
\midrule
\multirow{4}{*}{SpatialTSW} 
& Mapping & Map points onto new space & $O(nd_\theta)$ & $O(nd_\theta)$\\
& Projection & Matrix multiplication of points and lines & $O(Lknd_\theta)$ & $O(Lkd_\theta + nd_\theta)$\\
& Distance-based weight splitting & Distance calculation and softmax & $O(Lknd_\theta)$ & $O(Lkn + Lkd_\theta + nd_\theta)$ \\
& Sorting & Sorting projected coordinates & $O(Lkn\log n)$ & $O(Lkn)$ \\
& \textbf{Total} & & $O(Lknd_\theta + Lkn\log n)$ & $O(Lkn + Lkd_\theta + nd_\theta)$ \\
\midrule
\multirow{3}{*}{CircularTSW} 
& Circular projection & Subtraction and Norm calculation & $O(Lknd_\theta)$ & $O(Lkd_\theta + nd_\theta)$ \\
& Distance-based weight splitting & Distance calculation and softmax & $O(Lknd_\theta)$ & $O(Lkn + Lkd_\theta + nd_\theta)$ \\
& Sorting & Sorting projected coordinates & $O(Lkn\log n)$ & $O(Lkn)$ \\
& \textbf{Total} & & $O(Lknd_\theta + Lkn\log n)$ & $O(Lkn + Lkd_\theta + nd_\theta)$ \\
\midrule
\multirow{3}{*}{CircularTSW$_{r=0}$} 
& Circular projection & Subtraction and Norm calculation & $O(Lnd_\theta)$ & $O(Ld_\theta + nd_\theta)$ \\
& Distance-based weight splitting & Distance calculation and softmax & $O(Lknd_\theta)$ & $O(Lkn + Lkd_\theta + nd_\theta)$ \\
& Sorting & Sorting projected coordinates & $O(Ln\log n)$ & $O(Ln)$ \\
& \textbf{Total} & & $O(Lknd_\theta + Ln\log n)$ & $O(Lkn + Lkd_\theta + nd_\theta)$ \\
\midrule
\multirow{4}{*}{SpatialSTSW} 
& Mapping & Map points onto new space & $O(nd_\theta)$ & $O(nd_\theta)$\\
& Projection & Matrix multiplication of points and source & $O(Lnd_\theta)$ & $O(Ld_\theta + nd_\theta)$ \\
& Distance-based weight splitting & Distance calculation and softmax & $O(Lknd_\theta)$ & $O(Lkn + Lkd_\theta + nd_\theta)$ \\
& Sorting & Sorting projected coordinates & $O(Ln\log n)$ & $O(Ln)$ \\
& \textbf{Total} & & $O(Lknd_\theta + Ln\log n)$ & $O(Lkn + Lkd_\theta + nd_\theta)$ \\
\bottomrule
\end{tabular}
\end{adjustbox}
\end{small}
\end{center}
\vskip -0.1in
\end{table}

We provide complexity and memory analysis of our proposed distance. Since memory on GPU can be optimized for parallel processing capabilities, we provide the empirical memory usage on GPU, with expectation that our distance would be used in a GPU setting which is standard in machine learning.

\textbf{Computation and Memory Complexity.} Assuming $n \ge m$, the computational complexity of SpatialTSW and CircularTSW is $O(Lknd_\theta + Lkn\log n)$, while CircularTSW$_{r=0}$ and SpatialSTSW have a more efficient complexity of $O(Lknd_\theta + Ln\log n)$. All distances share an empirical memory cost of $O(Lkn + Lkd_\theta + nd_\theta)$. We analyze the main operations of our proposed distances in Table~\ref{tab:complexity_analysis}.

\textbf{Projection and Sorting.}  
In CircularTSW$_{r=0}$ and SpatialSTSW, the projected coordinates within a tree is the same for all lines. As a result, the computational complexity of these two steps in CircularTSW$_{r=0}$ and SpatialSTSW is reduced by a factor of the number of lines in a tree, $k$, compared to SpatialTSW and CircularTSW. This reduction is the primary reason for the computational advantage of CircularTSW$_{r=0}$ and SpatialSTSW.

\textbf{Memory Cost of Distance-Based Splitting.}  
As previously noted in prior work \cite{tran2025distancebased}, the empirical GPU-optimized memory cost of distance-based splitting is lower than its theoretical estimate due to kernel fusion optimizations. This operation consists of: (1) computing distance vectors from points to lines ($O(Lknd_\theta)$ computation and memory), (2) calculating their norms ($O(Lknd_\theta)$ computation and $O(Lknd_\theta)$ memory), and (3) applying softmax over all lines in each tree ($O(Lkn)$ computation and memory). While the theoretical cost is $O(Lknd_\theta)$ for both computation and memory, we leverage PyTorch’s automatic kernel fusion (via `torch.compile`) to merge these steps into a single operation. This enables the distance vectors ($Lkn \times d_\theta$) to be stored in shared GPU memory rather than global memory. As a result, only three matrices need to be stored: a line matrix ($O(Lkd_\theta)$), a support matrix ($O(nd_\theta)$), and a split weight matrix ($O(Lkn)$), reducing overall GPU memory usage to $O(Lkn + Lkd_\theta + nd_\theta)$.

\subsection{Runtime and memory analysis}
\label{appendix:runtime-memory-analysis}

\begin{figure}[h]
    \centering
    \begin{minipage}{0.49\textwidth}
        \centering
        \includegraphics[width=\textwidth]{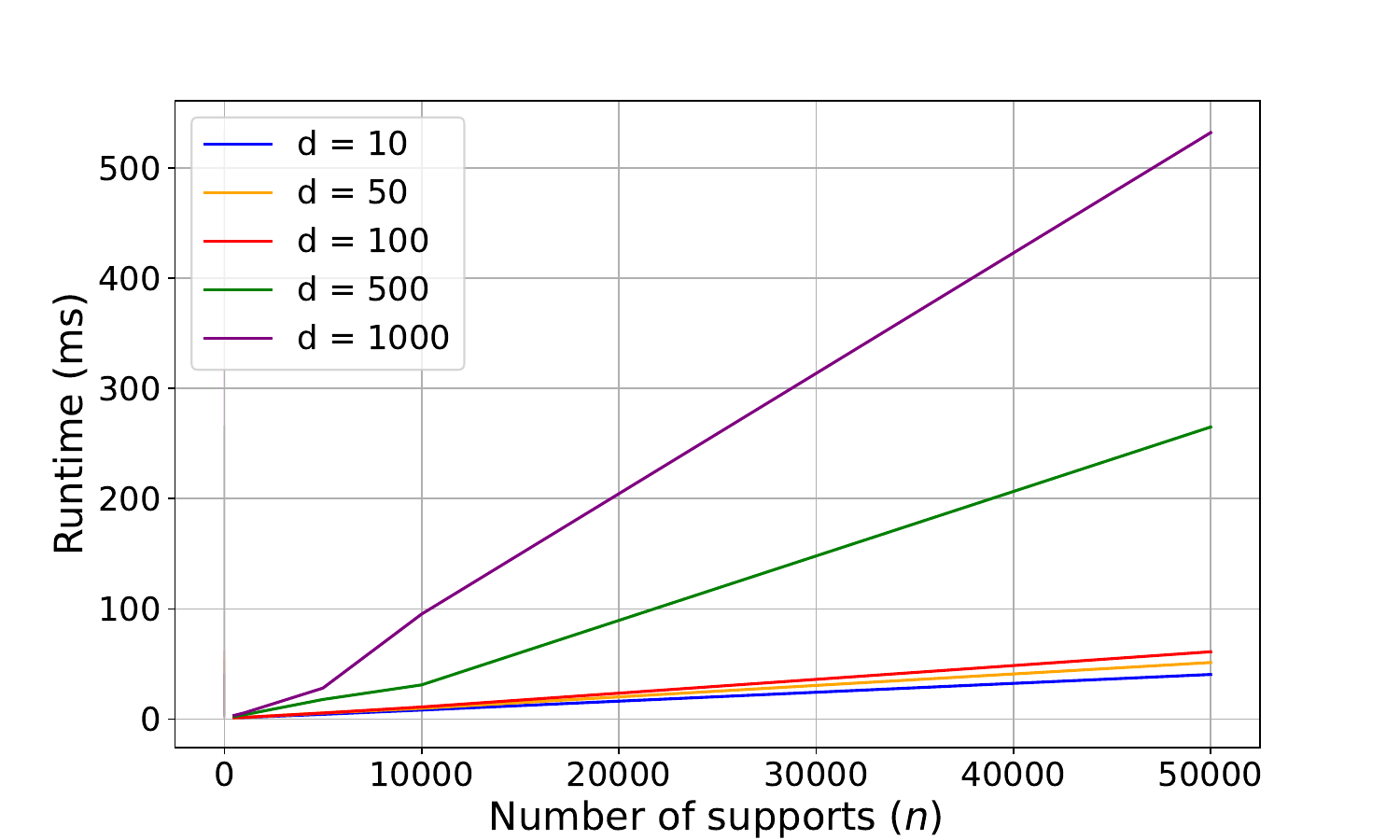}
    \end{minipage}
    \hfill
    \begin{minipage}{0.49\textwidth}
        \centering
        \includegraphics[width=\textwidth]{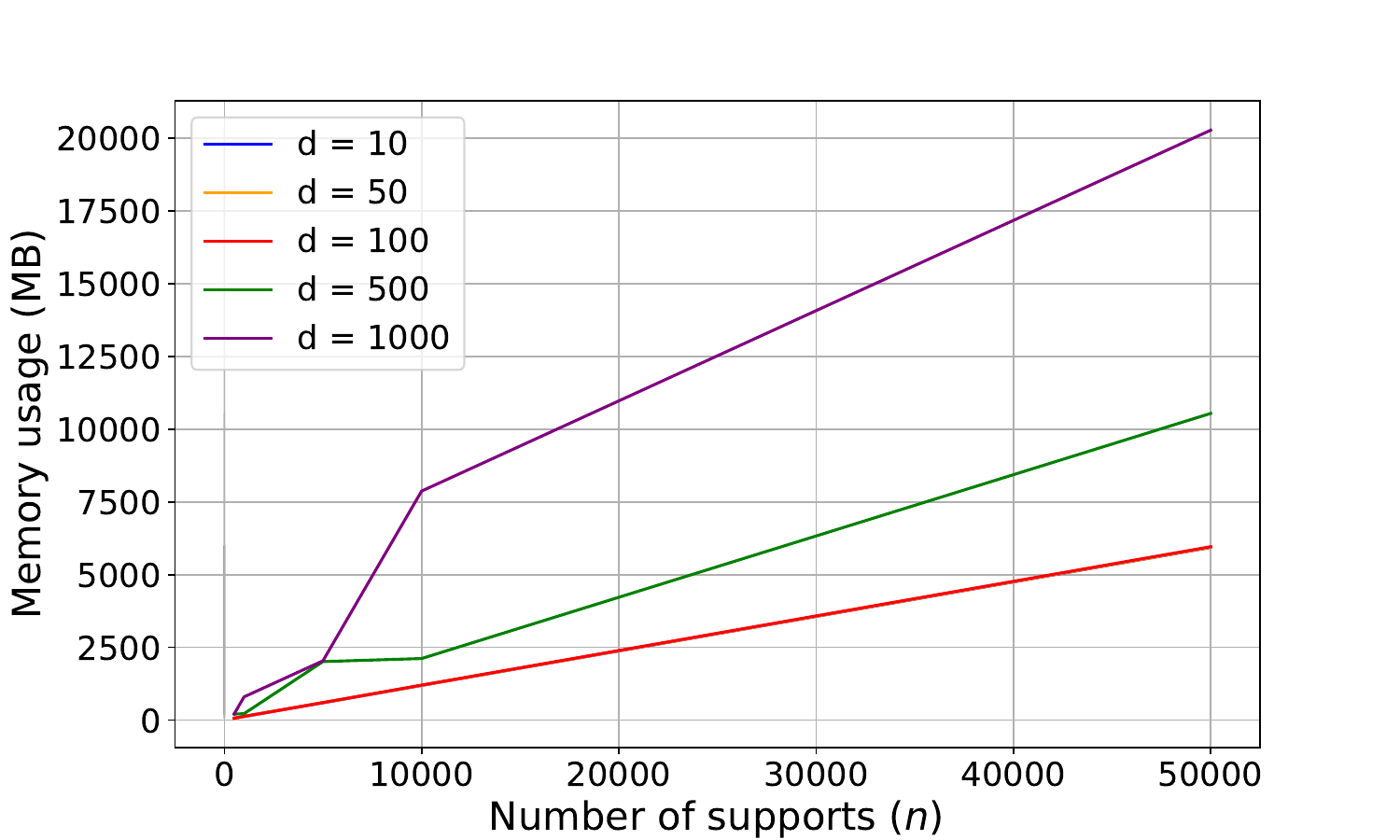}
    \end{minipage}
    
    \caption{Runtime and memory evolution of SpatialTSW.}
    \label{fig:pow_runtime_memory}
\end{figure}
\begin{figure}[h]
    \centering
    \begin{minipage}{0.49\textwidth}
        \centering
        \includegraphics[width=\textwidth]{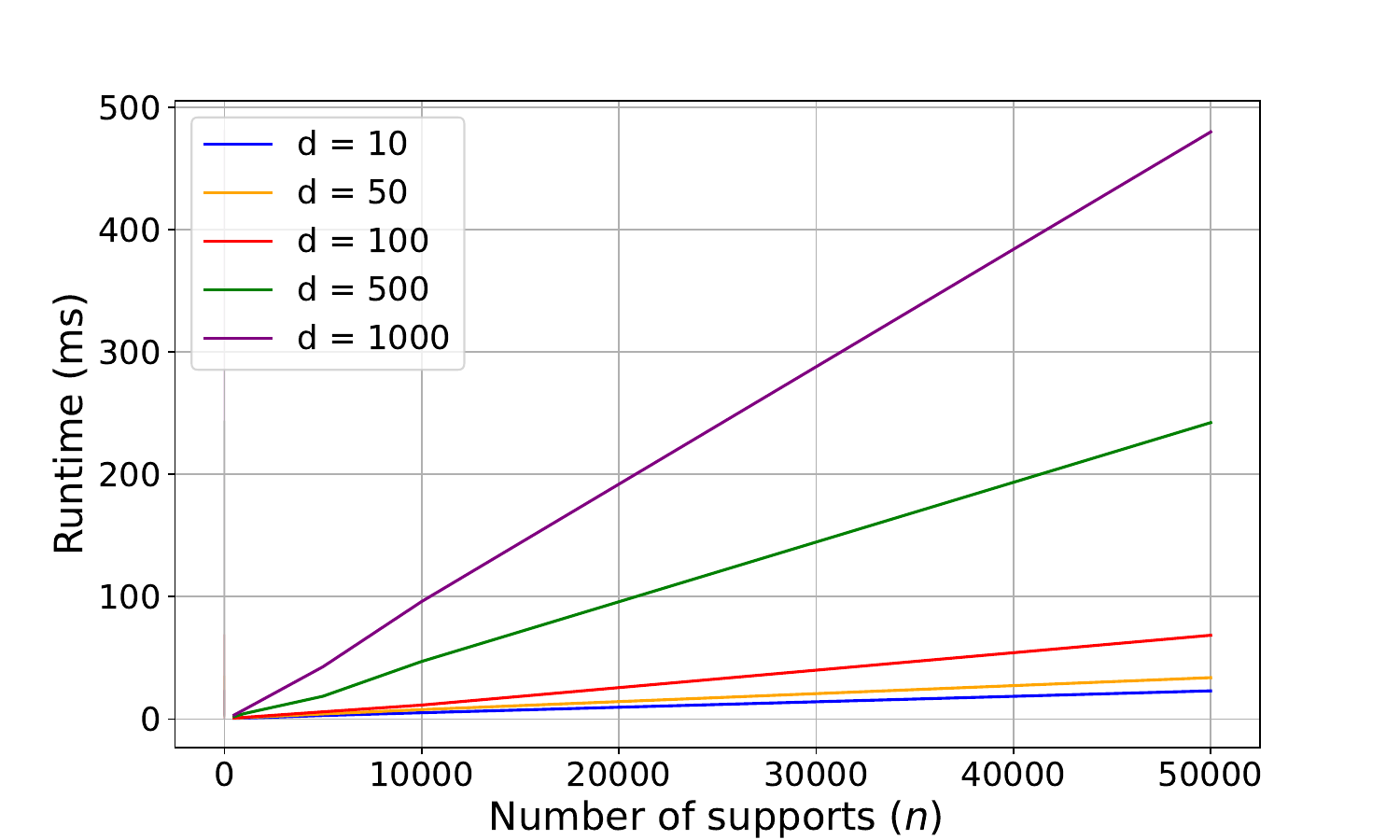}
    \end{minipage}
    \hfill
    \begin{minipage}{0.49\textwidth}
        \centering
        \includegraphics[width=\textwidth]{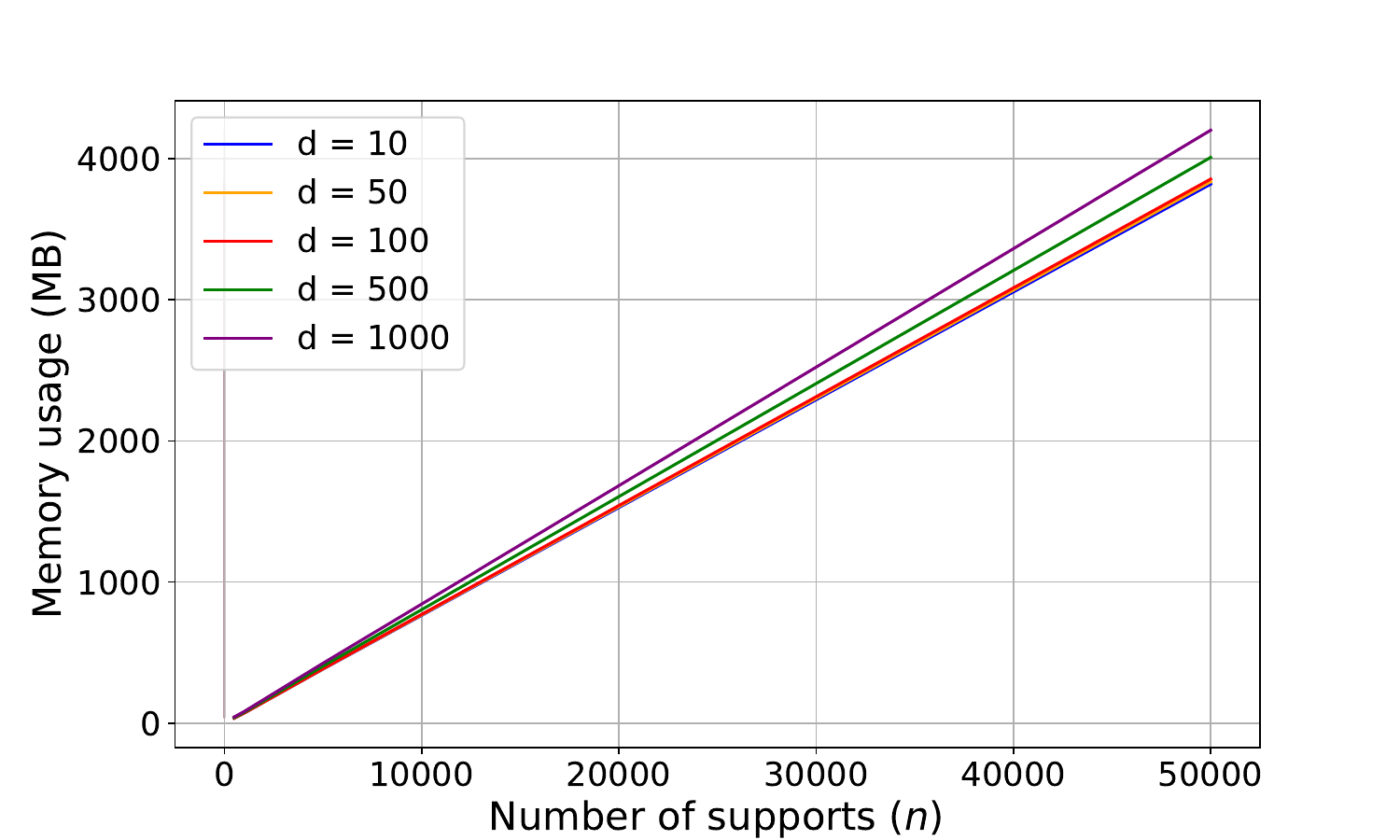}
    \end{minipage}
    
    \caption{Runtime and memory analysis of CircularTSW.}
    \label{fig:circular_runtime_memory}
\end{figure}

\begin{figure}[h]
    \centering
    \begin{minipage}{0.49\textwidth}
        \centering
        \includegraphics[width=\textwidth]{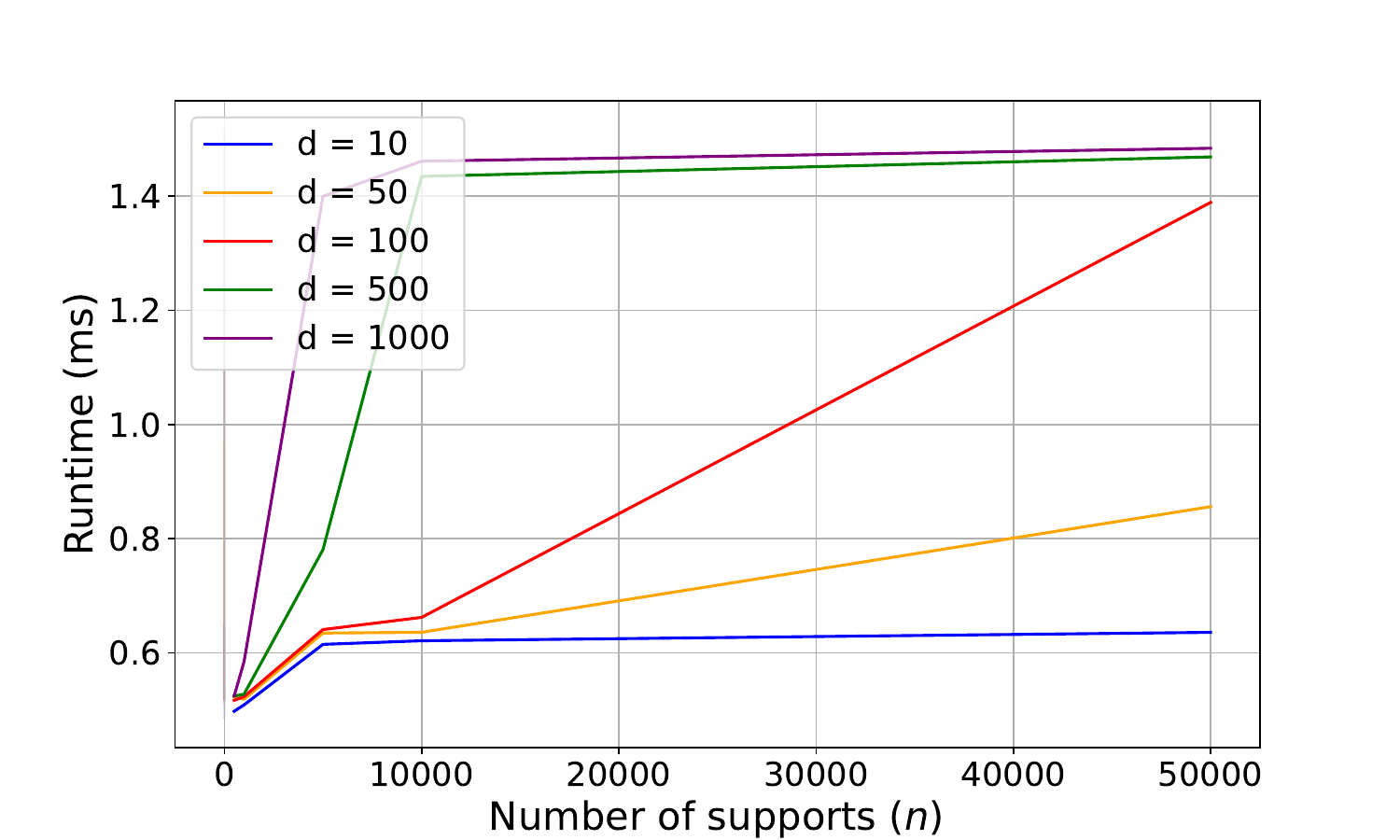}
    \end{minipage}
    \hfill
    \begin{minipage}{0.49\textwidth}
        \centering
        \includegraphics[width=\textwidth]{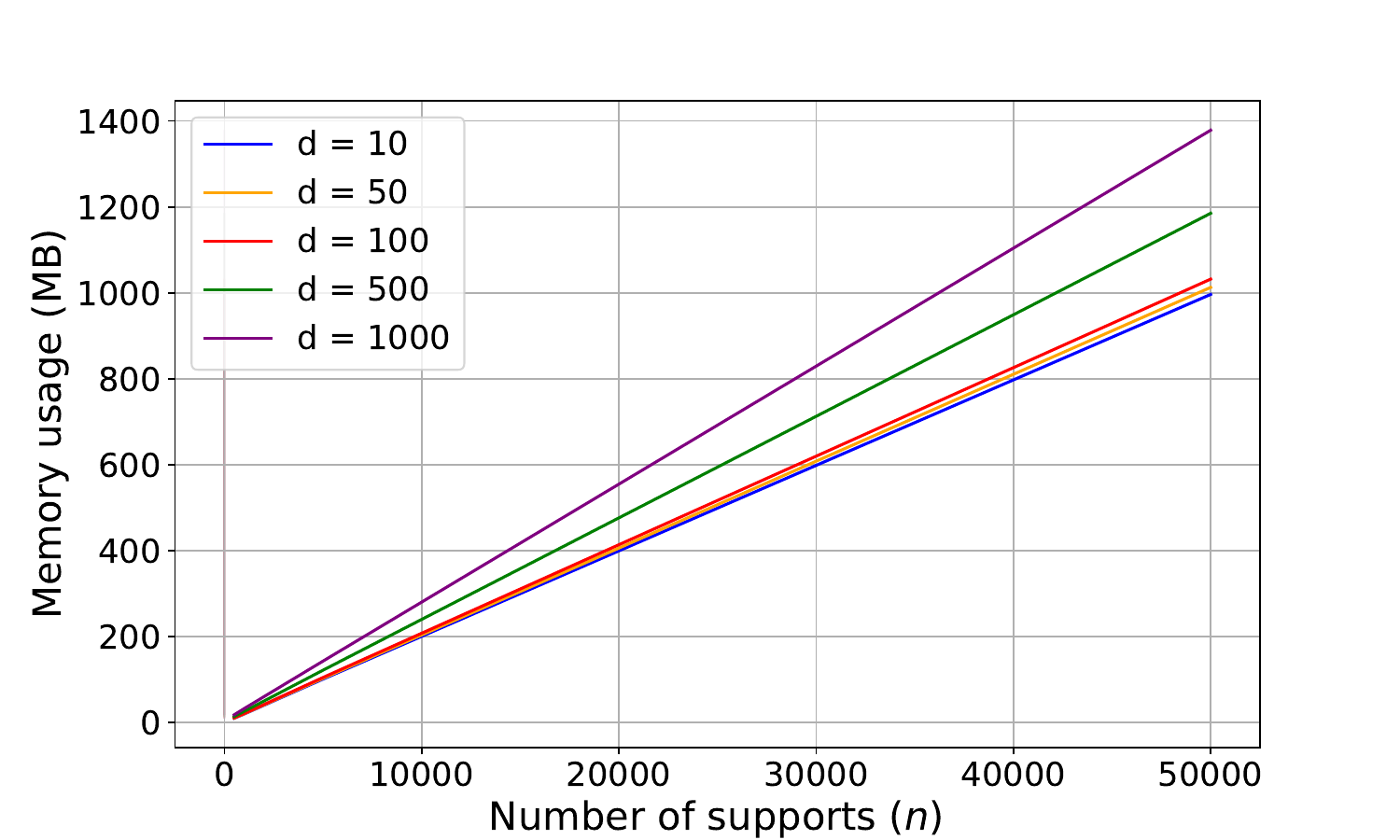}
    \end{minipage}
    
    \caption{Runtime and memory analysis of CircularTSW$_{r=0}$.}
    \label{fig:circular_concentric_runtime_memory}
\end{figure}

In this section, we conduct a runtime and memory analysis of SpatialTSW, CircularTSW, and CircularTSW$_{r=0}$ with respect to the number of supports and the support's dimension on a single NVIDIA A100 GPU. We fix $L = 2500$ and $k = 4$ (following the practical setting from our diffusion model experiments) for all configurations and vary $N \in \{500, 1000, 5000, 10000, 50000\}$ and $d \in \{10, 50, 100, 500, 1000\}$.    

\textbf{Runtime scaling.} Figures~\ref{fig:pow_runtime_memory} and~\ref{fig:circular_runtime_memory} demonstrate a linear relationship between the runtime of SpatialTSW and CircularTSW and the number of supports $n$. Regarding scaling with the data dimension, we observe that $d = 10000$ takes approximately twice the runtime of $d = 5000$, suggesting a linear relationship between dimension and computational time. This linear trend aligns with our theoretical complexity analysis in Appendix~\ref{appendix:complexity}. It is also worth noting that CircularTSW runs faster than SpatialTSW, as it relies on vector norms instead of vector multiplications. Regarding CircularTSW$_{r=0}$, Figure~\ref{fig:circular_concentric_runtime_memory} also demonstrates an almost linear relationship with the number of supports $n$. Interestingly, the runtime of CircularTSW$_{r=0}$ scales very efficiently with $d$. We suspect this is due to the reduced amount of vector normalization required compared to CircularTSW (by a factor of $k$). Ultimately, CircularTSW$_{r=0}$ is significantly faster than other Tree-Sliced methods, by two orders of magnitude when $d$ and $n$ are sufficiently large, which aligns with our theoretical complexity analysis. The significant reduction in computational cost when using the Circular Sliced Wasserstein variants arises from the efficiency of computing $L_2$ norms compared to inner products. This advantage is illustrated in Figure \ref{fig:inner_vs_l2_runtime}.

\begin{figure}[h]
\vskip 0.2in
\begin{center}
\centerline{\includegraphics[width=.7\columnwidth]{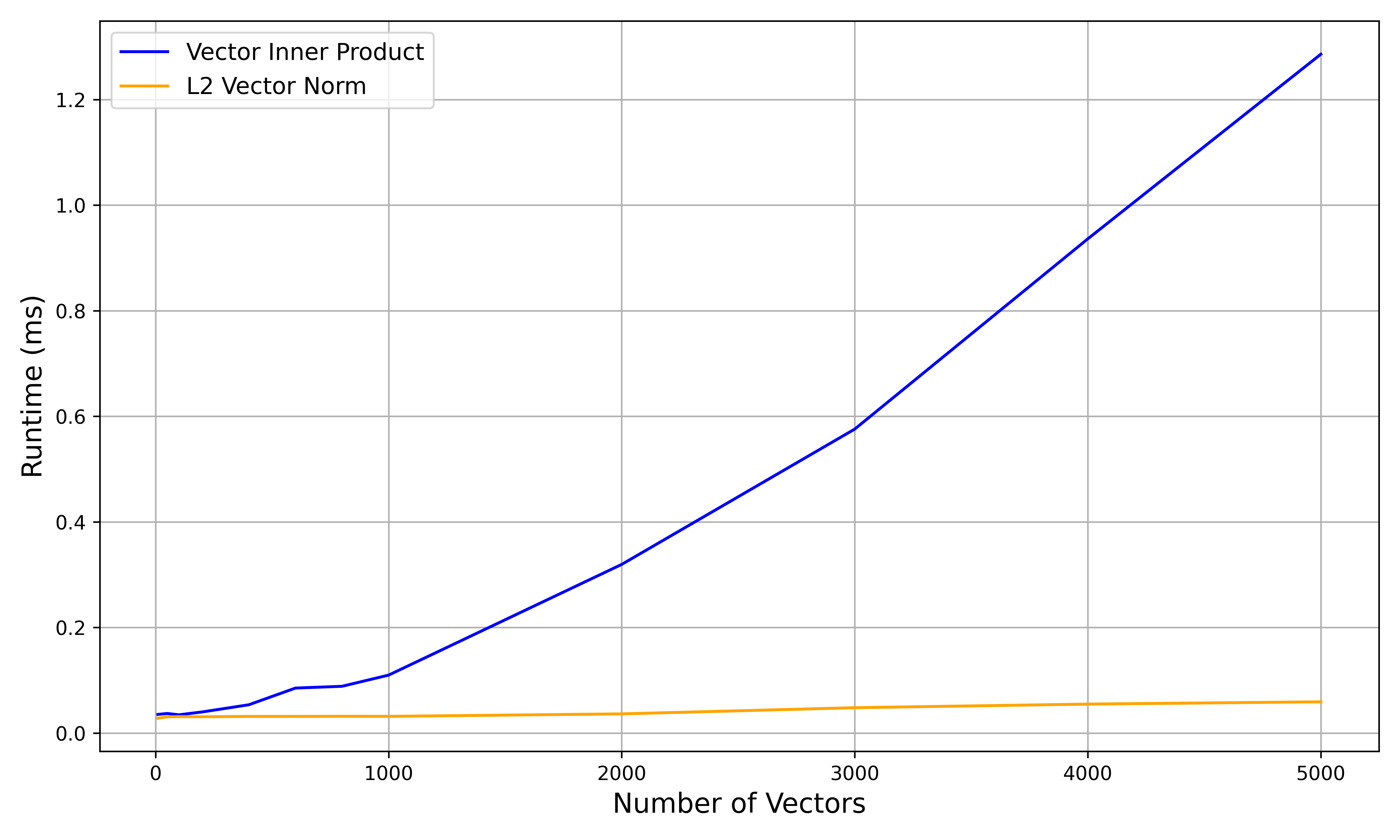}}
\caption{Rebuttal result comparing inner product and L2 norm performance, benchmarked on an A100 GPU with a fixed vector dimension of $d=3000$. The L2 norm demonstrates significantly faster computation times as the number of vectors increases.}
\label{fig:inner_vs_l2_runtime}
\end{center}
\vskip -0.2in
\end{figure}

\textbf{Memory scaling.} Figures~\ref{fig:pow_runtime_memory} ,~\ref{fig:circular_runtime_memory}, and~\ref{fig:circular_concentric_runtime_memory} presents the memory consumption analysis of SpatialTSW, CircularTSW, and CircularTSW$_{r=0}$, revealing a linear scaling relationship with $d$ and $n$. This aligns with the theoretical complexity analysis and suggests a predictable scaling behavior. Notably, the peak memory usage of CircularTSW and CircularTSW$_{r=0}$ is significantly lower than that of SpatialTSW.

\subsection{Splitting map for CircularTSW}
\label{appendix:splitting-map-circular-tsw}

We recall that the splitting map $\alpha$ is defined as:
\begin{align}
    \alpha(y,\mathcal{L})_l = \operatorname{softmax} \Bigl( \{ d(y,\mathcal{L})_i\}_{i=1}^k \Bigr),
\end{align}
where $d(y,\mathcal{L})_i$ represents the distance between $y$ and the $i^{\text{th}}$ line in $\mathcal{L}$. 

In the context of the Radon Transform on a system of lines, this involves computing the inner product between the support and the projection directions. However, in the Circular Radon Transform on a system of lines, the projection coordinate calculation involves the Euclidean norm $\|\cdot\|_2$. 

Therefore, we define a more computationally efficient distance function as:
\begin{align}
    d(y, \mathcal{L})_i = \left\| y - x_i - \left\| y - x_i - r\theta_i \right\|_2 \theta_i \right\|_2,
\end{align}
where $x_i$ is the source, $\theta_i$ is the direction of the $i^{\text{th}}$ line in $\mathcal{L}$, and $r$ is a fixed radius. 

Notably, this distance function still results in an $\operatorname{E}(d)$-invariant splitting map.

\subsection{Analysis on number of lines $k$ in CircularTSW$_{r=0}$}
\label{appendix:ablate-k-circular-concentric}

\begin{figure}[h]
\vskip 0.2in
\begin{center}
\centerline{\includegraphics[width=.7\columnwidth]{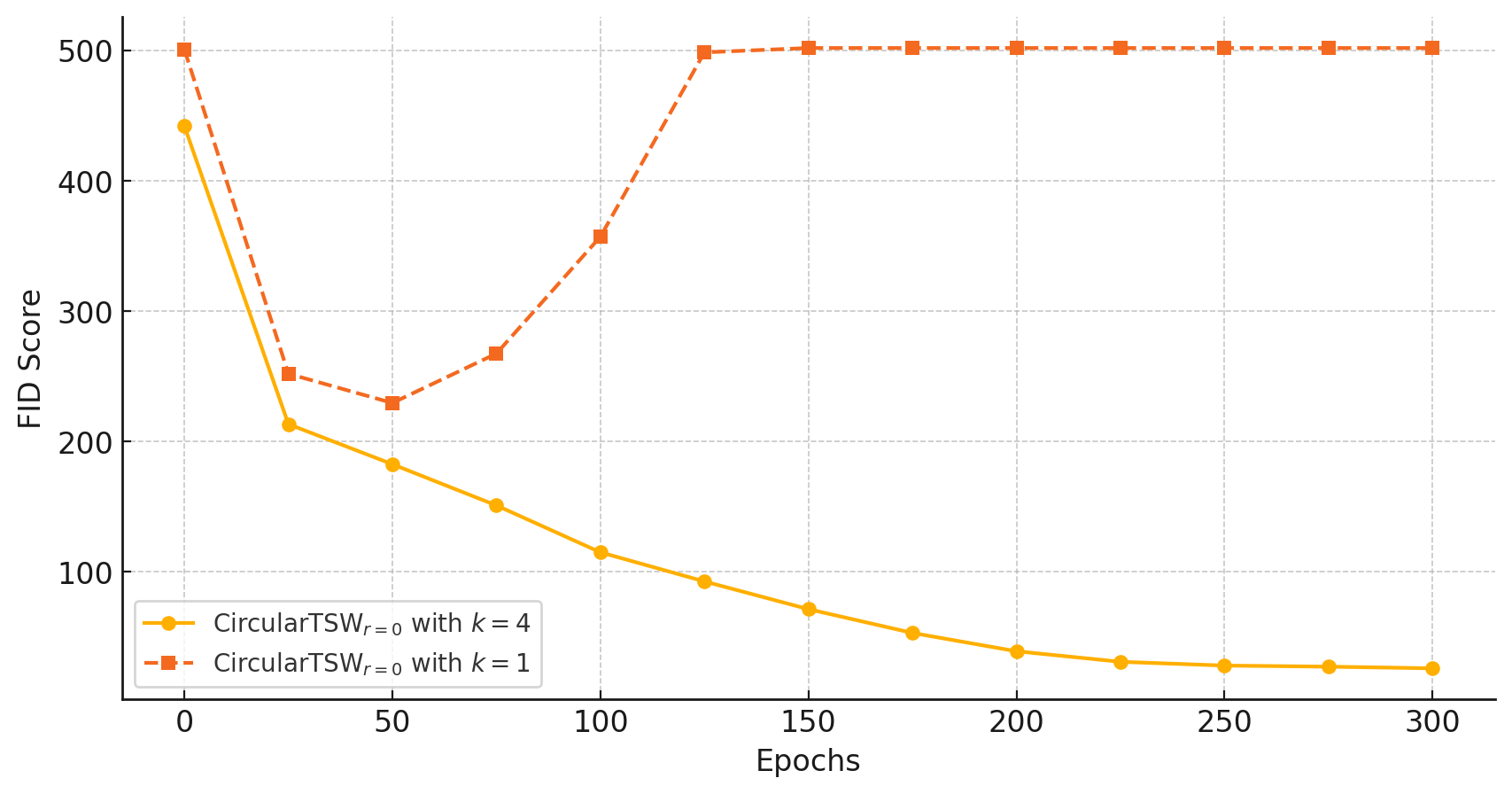}}
\caption{FID scores over the training process for CircularTSW-DD$_{r=0}$ with $k=1$ and $k=4$.}
\label{fig:circular-concentric-compare-k}
\end{center}
\vskip -0.2in
\end{figure}

In this section, we demonstrate that CircularTSW-DD$_{r=0}$ is effective only in a tree-based setting. To illustrate this, we conduct an experiment using the Denoising Diffusion Generative Adversarial Network (DDGAN), following the setup described in Appendix~\ref{appendix:exp-ddgan}. Figure~\ref{fig:circular-concentric-compare-k} presents the FID scores over 300 training epochs for models using CircularTSW$_{r=0}$ with $k=1$ and $k=4$. The results clearly show that only the model with $k=4$ trains stably, with its FID score gradually improving over time. In contrast, for $k=1$, the FID score suddenly spikes to 500 after epoch 125, indicating that the model collapses and starts generating meaningless images (all black pixels). This confirms that CircularTSW$_{r=0}$ is specifically designed for tree-like structures, relying on a distance-based splitting map to function effectively.

\subsection{Denoising Diffusion Generative Adversarial Network}
\label{appendix:exp-ddgan}

\paragraph{Diffusion Models.} 
Diffusion models \citep{sohl2015deep, ho2020denoising} have gained popularity as powerful generative models capable of producing high-quality data. In this experiment, we introduce their mechanisms and demonstrate the improvements introduced by our approach. The diffusion process begins with an initial sample from the distribution $q(x_0)$ and progressively corrupts it by adding Gaussian noise over $T$ steps. This process is formally defined as:

$$
q(x_{1:T} | x_0) = \prod_{t=1}^{T} q(x_t | x_{t-1}),
$$

where each transition follows a Gaussian distribution:

$$
q(x_t | x_{t-1}) = \mathcal{N}(x_t; \sqrt{1 - \beta_t} x_{t-1}, \beta_t I),
$$

with a predefined variance schedule $\beta_t$.

The objective of denoising diffusion models is to learn the reverse diffusion process, which reconstructs the original data from noisy samples. This requires estimating the parameters $\theta$ of the reverse process, formulated as:

$$
p_{\theta}(x_{0:T}) = p(x_T) \prod_{t=1}^{T} p_{\theta}(x_{t-1} | x_t),
$$

where each step follows a Gaussian transition:

$$
p_{\theta}(x_{t-1} | x_t) = \mathcal{N}(x_{t-1}; \mu_{\theta}(x_t, t), \sigma^2_t I).
$$

Training these models is typically done by maximizing the evidence lower bound (ELBO), which minimizes the Kullback-Leibler (KL) divergence between the true posterior and the model's approximation of the reverse diffusion process. This is expressed as:

$$
L = - \sum_{t=1}^{T} \mathbb{E}_{q(x_t)} \left[ \text{KL}(q(x_{t-1} | x_t) || p_{\theta}(x_{t-1} | x_t)) \right] + C,
$$

where $\text{KL}(\cdot || \cdot)$ represents the Kullback-Leibler divergence, and $C$ is a constant term.

\paragraph{Denoising Diffusion GANs.} 
While diffusion models generate high-quality and diverse samples, their slow sampling process limits real-world applicability. Denoising Diffusion GANs (DDGANs) \citep{xiao2021tackling} address this issue by modeling each denoising step using a multimodal conditional GAN, allowing for larger denoising steps. This significantly reduces the number of steps to just $4$, leading to sampling speeds up to $2000$ times faster than traditional diffusion models while maintaining competitive sample quality and diversity. The implicit denoising model in DDGANs is formulated as:

$$
p_{\theta}(x_{t-1} | x_t) = \int p_{\theta}(x_{t-1} | x_t, \epsilon) G_{\theta}(x_t, \epsilon) d\epsilon, \quad \epsilon \sim \mathcal{N}(0, I).
$$

\citet{xiao2021tackling} optimize the model parameters $\theta$ using adversarial training, with the objective:

$$
\min_{\phi} \sum_{t=1}^{T} \mathbb{E}_{q(x_t)}[D_{adv}(q(x_{t-1} | x_t) || p_{\phi}(x_{t-1} | x_t))],
$$

where $D_{adv}$ represents the adversarial loss. Instead, \citet{nguyen2024sliced} replace the adversarial loss with the Augmented Generalized Mini-batch Energy (AGME) distance. For two distributions $\mu$ and $\nu$, given a mini-batch size $n \geq 1$, AGME using a Sliced Wasserstein (SW) kernel is defined as:

$$
\text{AGME}^2_{b}(\mu, \nu; g) = \text{GME}^2_{b}(\tilde{\mu}, \tilde{\nu}),
$$

where $\tilde{\mu} = f_{\sharp} \mu$ and $\tilde{\nu} = f_{\sharp} \nu$, with $f(x) = (x, g(x))$ for a nonlinear function $g: \mathbb{R}^d \to \mathbb{R}$. The Generalized Mini-batch Energy (GME) distance \citep{salimans2018improving} is defined as:

$$
\text{GME}^{2}_b(\mu, \nu) = 2 \mathbb{E}[D(P_X, P_Y)] - \mathbb{E}[D(P_X, P_X')] - \mathbb{E}[D(P_Y, P_Y')],
$$

where $X, X' \overset{i.i.d.}{\sim} \mu^{\otimes m}$ and $Y, Y' \overset{i.i.d.}{\sim} \nu^{\otimes m}$, with

$$
P_X = \frac{1}{m} \sum_{i=1}^{m} \delta_{x_i}, \quad X = (x_1, \ldots, x_m).
$$

Here, $D$ represents any valid distance metric. In our work, we replace $D$ with Tree-Sliced Wasserstein (TSW) variants (our methods) and Sliced Wasserstein (SW) variants.

\paragraph{Setting.} 
We adopt the same architecture and hyperparameters as \citet{nguyen2024sliced} and \citet{tran2025distancebased}. Our models are trained for $1800$ epochs. For Tree-Sliced methods, including our own, we set $L = 2500$ and $k = 4$. For vanilla SW and SW variants, we follow \citet{nguyen2024sliced} and use $L = 10000$. The learning rate is also set according to \citet{nguyen2024sliced}, where $lr_d=1.25e\mathrm-4$ and $lr_g=1.6e\mathrm-4$. For SpatialTSW, we define $h(y) = y + y^3$, and for CircularTSW, we set $r = 0.01$. The standard deviation in tree sampling follows \citet{tran2025distancebased} and is set to $0.1$. To evaluate runtime, we use a batch size of 64 and measure time on a single NVIDIA A100 GPU.

\paragraph{Qualitative Results.} Figure~\ref{fig:qualitative_images} presents the qualitative results of SpatialTSW-DD, CircularTSW-DD, and Circular$_{r=0}$TSW-DD.

\begin{figure}[h]
    \centering
    \begin{minipage}{0.3\linewidth}
        \centering
        \includegraphics[width=\linewidth]{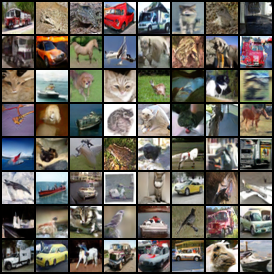}
        \\ SpatialTSW-DD
    \end{minipage}
    \begin{minipage}{0.3\linewidth}
        \centering
        \includegraphics[width=\linewidth]{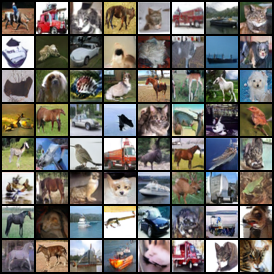}
        \\ CircularTSW-DD
    \end{minipage}
    \begin{minipage}{0.3\linewidth}
        \centering
        \includegraphics[width=\linewidth]{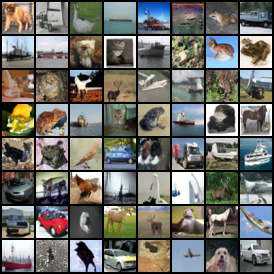}
        \\ Circular$_{r=0}$TSW-DD
    \end{minipage}
    \caption{Example images generated by our proposed DDGAN. The images correspond to (Left) SpatialTSW-DD, (Middle) CircularTSW-DD, and (Right) Circular$_{r=0}$TSW-DD.}
    \label{fig:qualitative_images}
\end{figure}

\subsection{Gradient Flow}
\label{appendix:exp-gradientflow}

\begin{table}[h]
\caption{Detailed results on the 25 Gaussians dataset. The table reports the average Wasserstein distance between source and target distributions over 5 runs.}
\label{table:gradientflow_appendix}
\vskip 0.15in
\begin{center}
\begin{small}
\begin{adjustbox}{width=0.95\textwidth}
\begin{tabular}{lcccccccc}
\toprule
\multirow{2}{*}{Methods}       & \multicolumn{5}{c}{Iteration}                                                                     & \multirow{2}{*}{Time/Iter($s$)} \\ 
\cmidrule(lr){2-6}
                                & 500              & 1000             & 1500             & 2000             & 2500             &           \\ \midrule
SW                              & 4.21e-1 $\pm$ 5.39e-3 & 1.54e-1 $\pm$ 2.43e-3         & 7.72e-2 $\pm$ 3.88e-3          & 4.97e-2 $\pm$ 3.30e-3          & 3.59e-2 $\pm$ 3.43e-3          & 0.0018     \\ 
MaxSW                           & 5.23e-1 $\pm$ 8.31e-3          & 2.36e-1 $\pm$ 4.63e-3         & 1.23e-1 $\pm$ 3.17e-3        & 8.04e-2 $\pm$ 3.70e-3        & 6.76e-2 $\pm$ 3.07e-3        & 0.1020      \\  
SWGG                            & 6.59e-1 $\pm$ 1.93e-2 & 3.62e-1 $\pm$ 2.70e-2 & 1.92e-1 $\pm$ 1.99e-2         & 9.07e-2 $\pm$ 1.31e-2          & 4.42e-2 $\pm$ 1.90e-2          & 0.0019     \\ 
LCVSW                           & \underline{3.46e-1 $\pm$ 4.63e-3}          & \underline{6.96e-2 $\pm$ 3.11e-3}          & 2.26e-2 $\pm$ 1.39e-3         & 1.31e-2 $\pm$ 2.07e-3         & 9.28e-3 $\pm$ 9.25e-4         & 0.0019    \\ 
\midrule
TSW-SL                          & 3.49e-1 $\pm$ 4.61e-3          & 8.10e-2 $\pm$ 2.34e-3          & \underline{1.06e-2 $\pm$ 1.00e-3}          & 2.68e-3 $\pm$ 3.24e-4          & 3.16e-6 $\pm$ 1.99e-6         & 0.0019     \\ 
Db-TSW                       & 3.50e-1 $\pm$ 5.10e-3         & 8.12e-2 $\pm$ 2.34e-3         & 1.09e-2 $\pm$ 1.41e-3 & \underline{1.77e-3 $\pm$ 6.69e-4} & \underline{1.30e-7 $\pm$ 9.28e-9} & 0.0020     \\
Db-TSW$^\perp$               & 3.52e-1 $\pm$ 5.17e-3         & 7.69e-2 $\pm$ 3.37e-3 & 2.73e-2 $\pm$ 4.87e-4 & 2.56e-3 $\pm$ 7.72e-4 & 2.03e-6 $\pm$ 3.70e-6 & 0.0021     \\
\midrule
SpatialTSW                          & \textbf{3.20e-1 $\pm$ 4.73e-3}          & \textbf{3.44e-2 $\pm$ 2.42e-3}          & \textbf{2.95e-3 $\pm$ 1.46e-4}          & \textbf{3.97e-4 $\pm$ 8.13e-5}          & \textbf{1.17e-7 $\pm$ 2.24e-8}          & 0.0021     \\ 
CircularTSW               & 4.28e-1 $\pm$ 4.33e-3         & 1.20e-1 $\pm$ 2.37e-3 & 3.48e-2 $\pm$ 5.10e-4 & 1.41e-2 $\pm$ 7.50e-4 &  7.86e-3 $\pm$ 3.94e-4 & 0.0017     \\
CircularTSW$_{r=0}$                       & 4.32e-1 $\pm$ 4.01e-3         & 1.22e-1 $\pm$ 1.50e-3         & 3.41e-2 $\pm$ 2.22e-3 & 1.45e-2 $\pm$ 1.03e-3 & 8.94e-3 $\pm$ 9.42e-4 & 0.0015     \\
\bottomrule
\end{tabular}
\end{adjustbox}
\end{small}
\end{center}
\vskip -0.1in
\end{table}

\paragraph{Detailed results on the \textit{25 Gaussians} dataset.} 
Table~\ref{table:gradientflow_appendix} presents the detailed results of our proposed methods and baselines. The low standard deviation of SpatialTSW indicates that our method consistently achieves faster convergence compared to other methods.

\paragraph{Ablation on $h ~ \colon ~ \mathbb{R}^d \rightarrow \mathbb{R}^{d_\theta}$ in SpatialTSW.}
\label{appendix:ablate-h} 
We ablate several injective continuous functions $h$ in SpatialTSW on the \textit{25-Gaussian} dataset. The results in Table~\ref{tab:ablate-h} show that, in general, $h(y) = y + \gamma y^3$ outperforms $h(y) = y + \gamma y^5$, although the latter tends to converge faster during the first 1500 iterations. The best result is achieved with $h(y) = y + 0.5 y^3$, yielding a final $W_2$ value of $9.59e\mathrm{-8}$.

\begin{table}[t]
\caption{Ablation study on the choice of $h$ in SpatialTSW for Gradient Flow. Results show the average Wasserstein distance between source and target distributions over 5 runs on the 25 Gaussians dataset.}
\label{tab:ablate-h}
\vskip 0.15in
\begin{center}
\begin{small}
\begin{adjustbox}{width=0.5\textwidth}
\begin{tabular}{llccccc}
\toprule
\multirow{2}{*}{$h(y)$} & \multirow{2}{*}{$\gamma$} & \multicolumn{5}{c}{Iteration} \\ 
\cmidrule(lr){3-7}
 &  & 500 & 1000 & 1500 & 2000 & 2500 \\ 
\midrule
\multirow{5}{*}{$y + \gamma y^3$} 
 & $0.1$  & 3.49e-1 & 7.12e-2 & 7.77e-3 & 6.49e-4 & \underline{1.15e-7} \\  
 & $0.5$  & 3.33e-1 & 4.68e-2 & 3.74e-3 & \textbf{1.58e-7} & \textbf{9.59e-8} \\  
 & $1$    & 3.20e-1 & 3.44e-2 & 2.95e-3 & 3.97e-4 & 1.17e-7 \\  
 & $5$    & 2.88e-1 & 3.26e-2 & 3.57e-3 & 3.39e-4 & 2.24e-7 \\  
 & $10$   & 2.73e-1 & 3.35e-2 & 5.11e-3 & 1.54e-4 & 5.59e-7 \\  
\midrule
\multirow{5}{*}{$y + \gamma y^5$} 
 & $0.1$  & 3.57e-1 & 7.29e-2 & 7.72e-3 & 1.71e-3 & 1.32e-7 \\  
 & $0.5$  & 3.36e-1 & 4.24e-2 & 4.57e-3 & 1.06e-3 & 1.50e-7 \\  
 & $1$    & 2.99e-1 & 2.37e-2 & 2.95e-3 & 3.50e-5 & 1.63e-7 \\  
 & $5$    & \underline{1.75e-1} & \underline{8.54e-3} & \underline{2.60e-3} & 2.03e-3 & 1.85e-3 \\  
 & $10$   & \textbf{1.34e-1} & \textbf{6.30e-3} & \textbf{3.14e-4} & \underline{6.40e-6} & 1.55e-6 \\  
\bottomrule
\end{tabular}
\end{adjustbox}
\end{small}
\end{center}
\vskip -0.1in
\end{table}

\paragraph{Hyperparameters.} 
For Tree-Sliced methods, we set $L = 25$ and $k = 4$. For the SW and SW-variant baselines, we use $L = 100$. The global learning rate is set to $0.001$. Each distribution in both datasets is sampled 500 supports.

\subsection{A Guide to Selecting Projection Variants}
Our motivation for proposing the non-linear projection framework is inspired by Generalized Sliced-Wasserstein (GSW) \cite{kolouri2019generalized}, which also includes both Circular and Spatial variants. It is underexplored in prior studies that among the three versions—original SW, SpatialSW, and CircularSW, which variant is most suitable for a given task.

This suggests that among the corresponding TSW variants, such as Db-TSW \cite{tran2025distancebased}, CircularTSW, and Spatial TSW, there is no guarantee that the versions with non-linear projections will consistently outperform the linear-projection TSW. However, the two new distance variants each offer distinct advantages over standard TSW, as outlined below:
\begin{itemize}
    \item The definition of SpatialTSW subsumes Db-TSW as a special case when the function is the identity map. This implies that models leveraging SpatialTSW have, in theory, greater representational capacity than those using Db-TSW. A similar relationship holds between the corresponding SW variants.
    \item The definition of CircularTSW is theoretically non-comparable to Db-TSW due to their fundamentally different constructions. However, CircularTSW$_{r=0}$ offers improved runtime efficiency. This benefit does not hold in the SW context, where CircularSW$_{r=0}$ performs poorly. One reason is that CircularSW$_{r=0}$ defines only a pseudo-metric, while CircularTSW$_{r=0}$ is a true metric.
\end{itemize}

Our framework offers greater flexibility by enabling a broader selection of distance functions. However, in Machine Learning, predicting the best variant for a task often requires empirical experimentation. Table \ref{table:full_result} shows that both Db-TSW and SpatialTSW perform well, but the non-linearity in SpatialTSW makes it hard to determine in advance which variant is better suited for a given task.

We offer intuition for selecting CircularTSW and CircularTSW$_{r=0}$. Since these distances rely on the $L_2$ norm for the projection step, they are likely to perform well when the $L_2$ norms of the data are diversely distributed. We validate this advantage over Db-TSW and SpatialTSW in the Table $\ref{table:dim_results_short}$, where the distribution of $L_2$ norms is uniform. We speculate that this property explains why CircularTSW performs effectively for the Diffusion experiment (Table \ref{table:diffusion}).

To the best of our knowledge, Db-TSW \cite{tran2025distancebased} is the only tree-sliced distance effectively suited for large-scale generative tasks involving transport from a training measure to a target measure in Euclidean space. Previously, \cite{tran2024tree} presents a basic and limited version of \cite{tran2025distancebased}, primarily emphasizing the constructive aspects of the tree-sliced approach, which serve as foundational groundwork. Meanwhile, \cite{tran2025spherical} explores the method in a spherical setting. Other works on Tree-Sliced Wasserstein (TSW), such as \cite{le2019tree}, \cite{yamada2022approximating}, and others, are mainly designed for classification tasks and are not applicable to generative settings. This limitation arises because these methods rely on a clustering-based framework for computing slices, which is theoretically unsuitable (as the clustering must be recomputed each time the training measure is updated, rendering previous clustering results irrelevant) and empirically inefficient (since clustering is significantly more computationally expensive than linear or non-linear projection methods).
\input{tables/rebuttal}

\paragraph{Hyperparameter $r$.} Selecting the optimal hyperparameter, such as $r$ for CircularTSW, is challenging and often requires empirical tuning. Intuitively, $r$ should be large enough to ensure diverse projections onto the lines but should not exceed the data's magnitude. For normalized data, we suggest starting with $r = \frac{1}{\sqrt{d}}$ and tuning from there.

\subsection{Spherical Gradient Flow}
\input{sections/spherical_exp/appendix_gf}

\subsection{Self-Supervised Learning}
\input{sections/spherical_exp/appendix_ssl}

\subsection{Sliced-Wasserstein Autoencoder (SWAE)}
\input{sections/spherical_exp/appendix_swae}

\subsection{Hardware settings}
\label{exp:resources}
The gradient flow experiments were conducted on a single NVIDIA A100 GPU, with each experiment taking approximately $0.5$ hours. The denoising diffusion experiments were executed in parallel on two NVIDIA A100 GPUs, with each run lasting around $50$ hours. All spherical experiments were conducted on a single NVIDIA A100 GPU.

%% file: tables/rebuttal.tex
\begin{table}[h]
\caption{Results for Tree-Sliced variants (Linear, Spatial, Circular) in a Gradient Flow task across datasets, showing the average Wasserstein distance between source and target distributions over 5 runs. Each method uses 100 projecting directions, trained for 500 iterations, with the best result reported over $lr \in \{1, 5 \times 10^{-1}, 1 \times 10^{-1}, 1 \times 10^{-2}, 5 \times 10^{-2}, 1 \times 10^{-3}, 3 \times 10^{-3}, 5 \times 10^{-3}\}$. SpatialTSW performs best on Half Moons, Swiss Roll, 25 Gaussians, and 8 Gaussians datasets, while Db-TSW excels on the Circle dataset.}
\label{table:full_result}
\vskip 0.1in
\begin{center}
\begin{small}
\begin{tabular}{l c c c c c}
\toprule
Methods & Circle & Half Moons & Swiss Roll & 25 Gaussians & 8 Gaussians \\
\midrule
SW               & 6.463e-4 $\pm$ 3.112e-5 & 1.648e-4 $\pm$ 3.754e-5 & 9.795e-4 $\pm$ 1.328e-4 & 4.007e-2 $\pm$ 1.940e-3 & 3.786e-2 $\pm$ 7.090e-3 \\
Db-TSW           & \textbf{1.331e-6 $\pm$ 1.123e-7} & \underline{1.659e-6 $\pm$ 1.471e-7}                    & \underline{1.659e-6 $\pm$ 1.471e-7} & \underline{2.103e-3 $\pm$ 6.378e-4} & \underline{3.475e-3 $\pm$ 1.151e-3} \\
\midrule
SpatialSW        & 2.098e-4 $\pm$ 1.919e-5 & 2.146e-4 $\pm$ 5.308e-5 & 9.329e-4 $\pm$ 1.113e-4 & 5.206e-2 $\pm$ 2.456e-3 & 3.593e-2 $\pm$ 2.619e-3 \\
SpatialTSW       & \underline{1.366e-6 $\pm$ 9.439e-8} & \textbf{1.267e-6 $\pm$ 6.638e-8} & \textbf{1.615e-6 $\pm$ 1.751e-7} & \textbf{1.969e-3 $\pm$ 6.314e-4} & \textbf{2.652e-3 $\pm$ 6.996e-4} \\
\midrule
CircularSW       & 6.044e-4 $\pm$ 1.670e-5 & 8.147e-5 $\pm$ 4.858e-6 & 7.950e-4 $\pm$ 1.065e-4 & 1.150e-1 $\pm$ 4.502e-3 & 2.137e-1 $\pm$ 1.276e-2 \\
CircularTSW      & 1.922e-4 $\pm$ 1.493e-5 & 6.982e-5 $\pm$ 4.253e-6 & 2.053e-4 $\pm$ 3.739e-5 & 1.172e-2 $\pm$ 8.564e-4 & 1.307e-2 $\pm$ 1.531e-3 \\
CircularTSW$_{r=0}$ & 7.924e-4 $\pm$ 5.153e-5 & 9.201e-5 $\pm$ 9.562e-6 & 4.030e-4 $\pm$ 5.877e-5 & 2.009e-2 $\pm$ 1.612e-3 & 3.044e-2 $\pm$ 9.105e-4 \\
\bottomrule
\end{tabular}
\end{small}
\end{center}
\vskip -0.1in
\end{table}

\begin{table}[h]
\caption{Results on the advantage of the Circular Tree-Sliced variant in a Gradient Flow task when the $L_2$ norm of the data is uniformly distributed. Data is sampled such that the $L_2$ norm follows $\text{Uniform}(0,1)$. The table reports the average Wasserstein distance between source and target distributions over 5 runs. Each method uses 100 projecting directions and is trained for 500 iterations, with the best result reported over $lr \in \{1, 5 \times 10^{-1}, 1 \times 10^{-1}, 1 \times 10^{-2}, 5 \times 10^{-2}, 1 \times 10^{-3}, 3 \times 10^{-3}, 5 \times 10^{-3}\}$. CircularTSW consistently achieves a lower Wasserstein distance, while Linear and Spatial variants struggle as the dimension $d$ increases.}
\label{table:dim_results_short}
\vskip 0.1in
\begin{center}
\begin{small}
\begin{tabular}{l c c c}
\toprule
Methods & $d = 2000$ & $d = 5000$ & $d = 10000$ \\
\midrule
SW                    & 0.535 $\pm$ 0.006   & 9.795 $\pm$ 0.025   & 70.06 $\pm$ 0.100   \\
Db-TSW        & 4.871 $\pm$ 0.049   & 87.49 $\pm$ 0.137    & 308.97 $\pm$ 0.367  \\
\midrule
SpatialSW             & 1.510 $\pm$ 0.006   & 18.66 $\pm$ 0.043    & 95.55 $\pm$ 0.170   \\
SpatialTSW            & 6.394 $\pm$ 0.066   & 93.08 $\pm$ 0.131    & 314.44 $\pm$ 0.269  \\
\midrule
CircularSW            & 0.357 $\pm$ 0.005   & \underline{0.404 $\pm$ 0.007}   & \underline{0.428 $\pm$ 0.004}  \\
CircularTSW           & \textbf{0.304 $\pm$ 0.009}   & \textbf{0.347 $\pm$ 0.010}   & \textbf{0.369 $\pm$ 0.015}  \\
CircularTSW$_{r=0}$   & \underline{0.332 $\pm$ 0.010}   & 0.517 $\pm$ 0.015   & 0.873 $\pm$ 0.022  \\
\bottomrule
\end{tabular}
\end{small}
\end{center}
\vskip -0.1in
\end{table}

%% file: sections/spherical_exp/appendix_gf.tex
\paragraph{Data.} Given the probability density function of the von Mises-Fisher distribution $f(x; \mu, \kappa) = C_{d}(\kappa)\exp(\kappa\mu^{T}x)$, where $\mu \in \mathbb{S}^{d}$ is mean direction and $\kappa > 0$ is concentration parameter and the normalization constant $C_{d}(\kappa) = \dfrac{\kappa^{d/2 - 1}}{(2\pi)^{p/2}I_{p/2 - 1}(\kappa)}$, we use 12 vMFs as the target distribution with $\kappa = 50$ and mean directions as follows:
\[
\begin{array}{cccc}
\mu_1 = (-1, \phi, 0), & \mu_2 = (1, \phi, 0), & \mu_3 = (-1, -\phi, 0), & \mu_4 = (1, -\phi, 0) \\
\mu_5 = (0, -1, \phi), & \mu_6 = (0, 1, \phi), & \mu_6 = (0, -1, -\phi), & \mu_8 = (0, 1, -\phi) \\
\mu_9 = (\phi, 0, -1), & \mu_{10} = (\phi, 0, 1), & \mu_{11} = (-\phi, 0, -1), & \mu_{12} = (-\phi, 0, 1), \\
\end{array}
\]
where $\phi = \dfrac{1 + \sqrt{5}}{2}$.

Similar to \citet{tran2024stereographic, tran2025spherical}, we pick 200 samples from each vMF.

\paragraph{Setting.} We set $L = 200$ trees and $k = 5$ lines for STSW and SpatialSTSW, and $L = 1000$ projections for other sliced methods. ARI-S3W (30) employs 30 rotations with a pool size of 1000, whereas RI-S3W (1) and RI-S3W (5) use 1 and 5 rotations, respectively. Training is conducted using Adam \citep{kinga2015method} optimizer with learning rate $lr = 0.01$, following update rules \cite{bonet2022spherical}: 
\[
\begin{cases}
x^{(k+1)} = x^{(k)} - \gamma \nabla_{x^{(k)}} \text{SpatialSTSW}(\hat{\mu}_k, \nu), \\
x^{(k+1)} = \frac{x^{(k+1)}}{\| x^{(k+1)} \|_2}.
\end{cases}
\]

%% file: sections/spherical_exp/appendix_ssl.tex
\paragraph{Encoder.} Following the approach outlined in \citet{bonet2022spherical, tran2024stereographic, tran2025spherical}, we use ResNet18 \citep{he2016deep} as the encoder. We train it on CIFAR-10 data for 200 epochs with a batch size of 512 and the SGD optimizer with initial $lr = 0.05$, momentum $0.9$, and weight decay $10^{-3}$. To generate positive pairs, we employ commonly used augmentation techniques, consistent with previous studies \citep{wang2020understanding, bonet2022spherical, tran2024stereographic, tran2025spherical}. These transformations include resizing, cropping, horizontal flipping, color jittering, and random grayscale conversion.

For STSW and SpatialSTSW, we configure $L = 200$ trees and $k = 20$. For other distances, we use $L = 200$ projections and $N_R = 5$ with a pool size of 100 for RI-S3W and ARI-S3W. The regularization coefficients are chosen as follows: $\lambda = 10$ for STSW and SpatialSTSW, $\lambda = 1$ for SW, $\lambda = 20$ for SSW, and $\lambda = 0.5$ for S3W variants.

\paragraph{Linear Classifier.} To evaluate the quality of the feature representations learned by the pre-trained encoder, a linear classifier is trained on top of these features. Following the approach of \citet{bonet2022spherical}, the classifier is trained for 100 epochs using the Adam \citep{kinga2015method} optimizer. We set the initial learning rate to $10^{-3}$ together with a weight decay of $0.2$ at epochs $60$ and $80$.

%% file: sections/spherical_exp/appendix_swae.tex
\paragraph{Setup.} For our model training, we use Adam \citep{kinga2015method} optimize, setting learning rate 
 to $lr = 10^{-3}$. The model undergoes training over 100 epochs with a batch size of 500, where the binary cross-entropy (BCE) loss function serves as the reconstruction loss. For STSW and SpatialSTSW, we fix $L = 200$ trees and $k = 10$ lines. Other sliced methods use $L=100$ projections. RI-S3W ad ARI-S3W use $N_R = 5$ rotation and pool size of 100. We use prior 10 vMFs, while setting the regularization parameter $\lambda = 1$ for STSW, SpatialSTSW, $\lambda = 10$ for SSW, and $\lambda = 10^{-3}$ for SW and S3W variants.

\textbf{CIFAR-10 Model Architecture.} \cite{tran2024stereographic, tran2025spherical}

Encoder:
\begin{align*}
    x \in \mathbb{R}^{3 \times 32 \times 32} &\rightarrow \text{Conv2d}_{32} \rightarrow \text{ReLU} \rightarrow \text{Conv2d}_{32} \rightarrow \text{ReLU} \\
    &\rightarrow \text{Conv2d}_{64} \rightarrow \text{ReLU} \rightarrow \text{Conv2d}_{64} \rightarrow \text{ReLU} \\
    &\rightarrow \text{Conv2d}_{128} \rightarrow \text{ReLU} \rightarrow \text{Conv2d}_{128} \rightarrow \text{Flatten}\\
    &\rightarrow \text{FC}_{512} \rightarrow \text{ReLU} \rightarrow \text{FC}_{3}\\
    &\rightarrow \ell^2 \text{ normalization} \rightarrow z \in \mathbb{S}^2
\end{align*}
Decoder:
\begin{align*}
    z \in \mathbb{S}^2 &\rightarrow \text{FC}_{512} \rightarrow \text{FC}_{2048} \rightarrow \text{ReLU}\\
    &\rightarrow \text{Reshape}(128 \times 4 \times 4) \rightarrow \text{Conv2dT}_{128} \rightarrow \text{ReLU}\\
    &\rightarrow \text{Conv2dT}_{64} \rightarrow \text{ReLU} \rightarrow \text{Conv2dT}_{64} \rightarrow \text{ReLU}\\
    &\rightarrow \text{Conv2dT}_{32} \rightarrow \text{ReLU} \rightarrow \text{Conv2dT}_{32} \rightarrow \text{ReLU}\\
    &\rightarrow \text{Conv2dT}_{3} \rightarrow \text{Sigmoid}
\end{align*}
